\documentclass[10pt,twocolumn,letterpaper]{article}

\usepackage{cvpr}
\usepackage{graphicx}
\usepackage{comment}
\usepackage{amsmath,amssymb} 
\usepackage{color}
\usepackage{times}
\usepackage{epsfig}

\usepackage{amsfonts}       
\usepackage{amsthm}   

\newcommand{\R}{{\mathbb{R}}}

\newcommand{\A}{{\mathcal{A}}}
\newcommand{\E}{{\mathbb{E}}}
\newcommand{\be}{\begin{eqnarray}}
  \newcommand{\ee}{\end{eqnarray}}
\newcommand{\beq}{\begin{equation}\begin{aligned}}
  \newcommand{\eeq}{\end{aligned}\end{equation}}
\newcommand{\beqn}{\begin{equation*}\begin{aligned}}
  \newcommand{\eeqn}{\end{aligned}\end{equation*}}
\newcommand{\ben}{\begin{eqnarray*}}
  \newcommand{\een}{\end{eqnarray*}}

\newtheorem{prop}{Proposition}
\newtheorem{lemma}{Lemma}

\usepackage{algorithm}
\usepackage{algorithmic}

\usepackage{multirow}
\usepackage{booktabs}       

\usepackage[pagebackref=true,breaklinks=true,letterpaper=true,colorlinks,bookmarks=false]{hyperref}

\usepackage{authblk}
\cvprfinalcopy
\ifcvprfinal\fi
\begin{document}

\title{Blind Adversarial Training: \\Balance Accuracy and Robustness} 
\author[1]{Haidong Xie}
\author[1]{Xueshuang Xiang\thanks{Corresponding author: xiangxueshuang@qxslab.cn}}
\author[1]{Naijin Liu}
\author[2,3,4]{Bin Dong}
\affil[1]{\normalsize Qian Xuesen Laboratory of Space Technology, China Academy of Space Technology}
\affil[2]{\normalsize Beijing International Center for Mathematical Research, Peking University}
\affil[3]{\normalsize Center for Data Science, Peking University}
\affil[4]{\normalsize Beijing Institute of Big Data Research}

\date{}

\maketitle

\begin{abstract}
 Adversarial training (AT) aims to improve the robustness of deep learning models by mixing clean data and adversarial examples (AEs). Most existing AT approaches can be grouped into restricted and unrestricted approaches. Restricted AT requires a prescribed uniform budget to constrain the magnitude of the AE perturbations during training, with the obtained results showing high sensitivity to the budget. On the other hand, unrestricted AT uses unconstrained AEs, resulting in the use of AEs located beyond the decision boundary; these overestimated AEs significantly lower the accuracy on clean data. These limitations mean that the existing AT approaches have difficulty in obtaining a comprehensively robust model with high accuracy and robustness when confronting attacks with varying strengths. Considering this problem, this paper proposes a novel AT approach named blind adversarial training (BAT) to better balance the accuracy and robustness. The main idea of this approach is to use a cutoff-scale strategy to adaptively estimate a nonuniform budget to modify the AEs used in the training,  ensuring that the strengths of the AEs are dynamically located in a reasonable range and ultimately improving the overall robustness of the AT model. The experimental results obtained using BAT for training classification models on several benchmarks demonstrate the competitive performance of this method.
\end{abstract}

\section{Introduction}
\label{sec:intro}

Deep learning~\cite{Hinton_Deep_learning,Goodfellow_2016} has made great breakthroughs in many fields, such as computer vision~\cite{Krizhevsky2012ImageNet}, speech recognition~\cite{Mikolov2012Strategies,Hinton2012Deep}, and natural language processing~\cite{Sutskever2014Sequence}. However, after \textbf{adversarial examples (AEs)} were introduced~\cite{szegedy2013intriguing,2014arXiv1412.6572G}, the weakness of deep neural networks has attracted  increasingly more attention. Many effective AE generation methods and defensive strategies are proposed; see the review papers and references therein for details~\cite{akhtar2018threat,Zhang2018Adversarial}.

\begin{figure}[t]
  \centering
  \includegraphics[width=\linewidth]{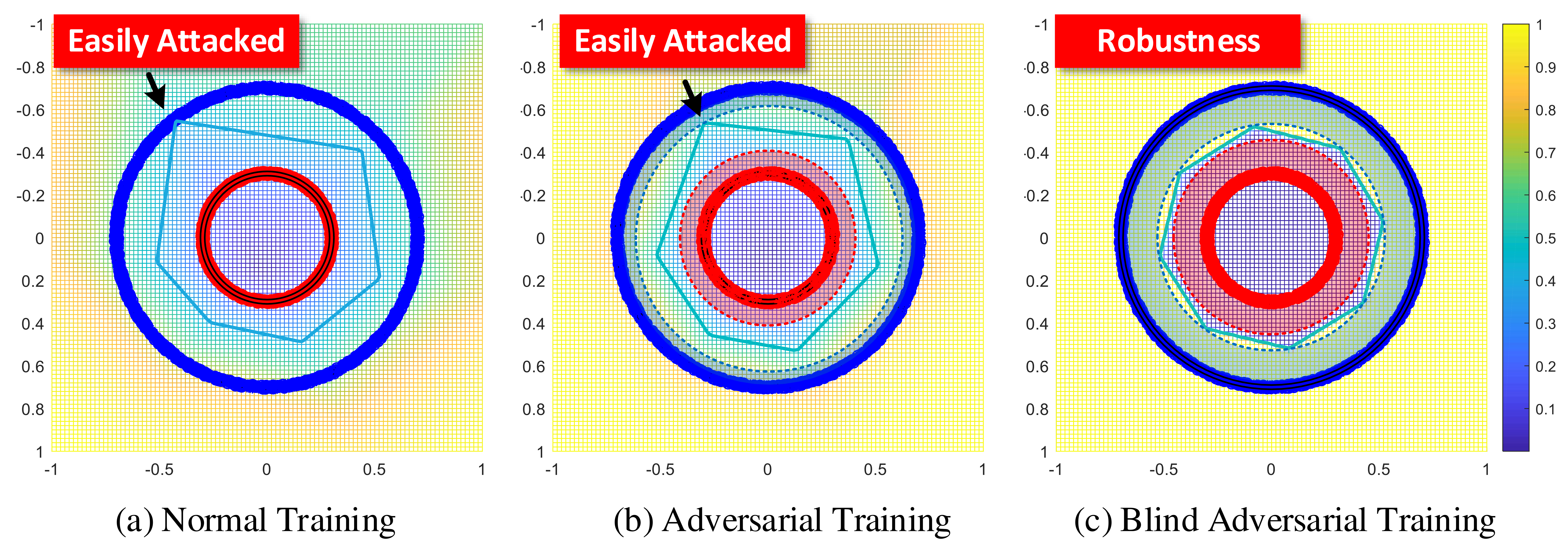}
  \caption{Comparison of different training approaches employing the two circles classification problem with the $1$-hidden-layer $6$-dim perceptron. We show the results of AT with a budget of $0.1$. The solid blue/red lines correspond to the datasets of two labels, the solid gray lines represent the decision boundary of the classifiers, and the blue/red shadow zones in (b,c) show the manifold of AEs.}
  \label{fig:tcc}
\end{figure}

\begin{figure*}[t]
  \centering
  \includegraphics[width=1\linewidth]{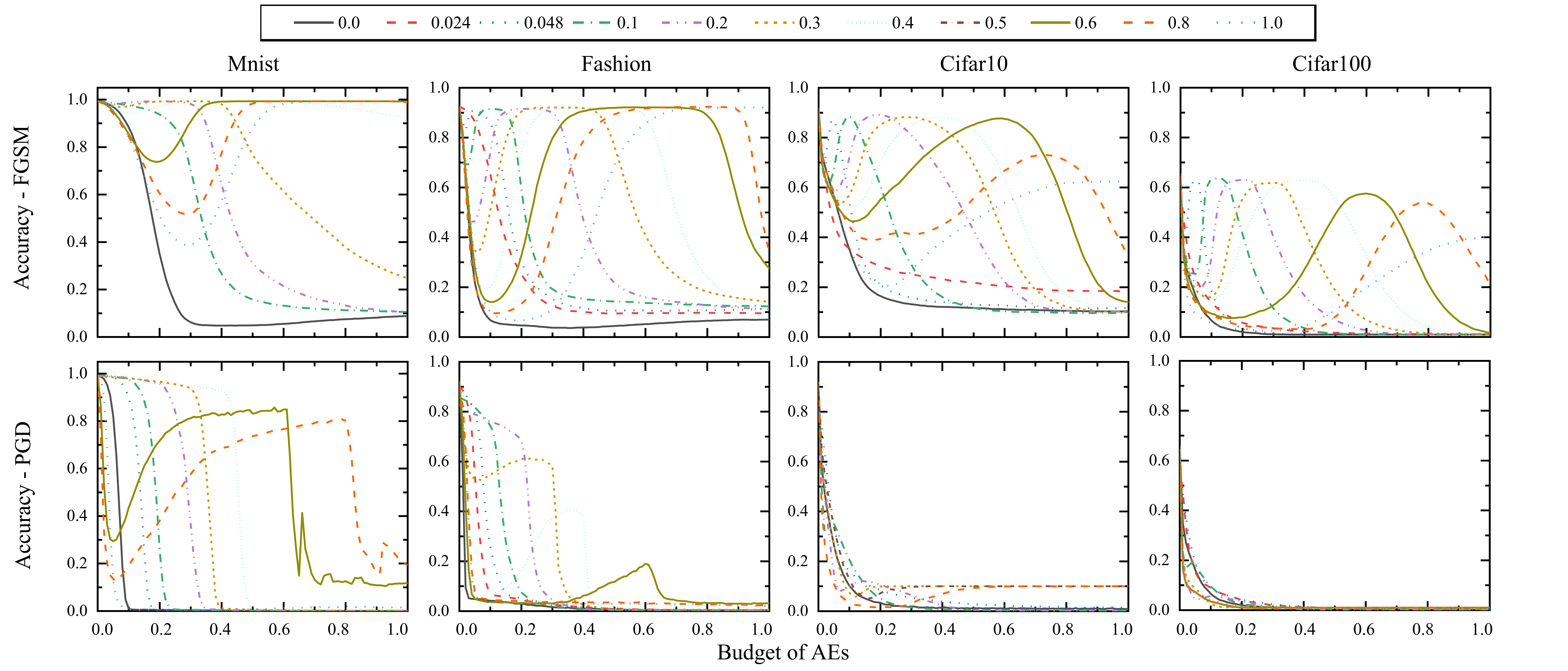}
  \caption{Comparison of the accuracy against attacks with varying budgets (x-axis) of the AT models (FGSM-AT/PGD-AT) with varying budgets (the lines as the legend) on MNIST, Fashion-MNIST, CIFAR10 and CIFAR100. For more results, see Appendix~C.}
  \label{fig:fgsm_pgd}
\end{figure*}

\textbf{Adversarial training (AT)} is a process of training a neural network on a mixture of clean data and AEs in order to improve the robustness of the network against adversarial attacks; see~\cite{szegedy2013intriguing,2014arXiv1412.6572G,madry2017towards} for white-box attacks and~\cite{kurakin2016adversarialb,tramer2017ensemble,song2018improving} for black-box attacks. These approaches focus on improving the generalization ability of AT by modifying the AE generation method or training loss. However, the AEs perturbations during training is the main problem in this approach. Most existing AT approaches can be divided into two categories, restricted and unrestricted AT. For both of the approaches, it was proved numerically~\cite{madry2017towards} and theoretically~\cite{Tsipras2018} that an improvement in the robustness is always accompanied by a loss of accuracy. 
Furthermore, the limitation of incurring a minimal impact on accuracy requires the architecture of the network to be sufficiently expressive~\cite{akhtar2018threat,Zhang2018Adversarial}, as guaranteed by the universal approximation theorem~\cite{2014arXiv1412.6572G,hornik1989multilayer} and regularization~\cite{Sank_2017}; even so, it is still difficult to achieve improved robustness~\cite{2014arXiv1412.6572G,tramer2017ensemble,7546524}. Therefore, we first analyze the characteristics of existing AT approaches.

The restricted AT approaches use norm-constrained AEs, such as FGSM-AT~\cite{2014arXiv1412.6572G} and PGD-AT~\cite{madry2017towards} with using FGSM and PGD AEs respectively, require a prescribed uniform budget to constrain the magnitude of the AE perturbation during training and then evaluate the result on the AEs with the same budget. The model obtained with the prescribed budget is only robust when confronting the attack with the same strength, while it is clearly weak for an attack stronger than the prescribed budget, and is overly defensive when confronting a small attack or encountering clean data. Madry \etal.~\cite{madry2017towards} found that a large budget is necessary to improve the AT effectiveness, but possibly lowering the accuracy on clean data, this result was also verified numerically by Song \etal.~\cite{song2018improving}. We analyze this issue using a synthetic problem, namely, the two circles classification (TCC) problem~\cite{website:Dplay}. As shown in Figure~\ref{fig:tcc}~(a), a classifier employing normal training (NT) can be trained with an accuracy of $100\%$ but easily attacked. For AT with a budget of $0.1$, as shown in Figure~\ref{fig:tcc}~(b), we obtain a more robust model that can only defend against the attacks with a strength less than the budget of $0.1$; this model is still easily attacked when encountering a larger attack. Furthermore, it is found from TCC that the AT with excessive budget cannot ensure the robustness; this is attributed to the AEs from different labels going beyond the decision boundary and overlapping or touching each other, preventing the model updating for greater robustness. Moreover, the results of FGSM-AT and PGD-AT on several datasets also show the deficiencies of the prescribed budget, as shown in Figure~\ref{fig:fgsm_pgd} and Appendix~C. 

By contrast, the unrestricted AT approaches are not affected by the prescribed budget and use unconstrained AEs with enough perturbations (beyond the decision boundary). To date, many types of unconstrained AEs have been examined; for example, DeepFool~\cite{moosavi2016deepfool} attempts to obtain the smallest AEs aiming for the decision boundary, and CW~\cite{cw} can obtain AEs that balance the perturbation and confidence. While we can directly apply these AEs to AT, the basic motivation of unconstrained AEs is to attack while aiming to fool the model by using the AEs beyond the decision boundary. DDN-AT~\cite{DBLP:journals/corr/abs-1811-09600} decouples the direction and norm of gradient-based attacks while inheriting the advantages of high computational efficiency, while MMA training~\cite{DBLP:journals/corr/abs-1812-02637}, focuses on maximizing the margins, which is an alternative approach for selecting the budget for each point individually. These methods inherit the perturbation of AEs and change the norm toward the decision boundary. Based on the obtained numerical experience,  the  loss function gradients are unstable in this case, leading to  dramatic decision boundary fluctuations, and giving rise to a very large number of AEs with too large strength; this severely decreases the robustness and lowers the accuracy on clean data. Furthermore, the discussion below indicates that the unconstrained AEs on the decision boundary cannot ensure the best robustness.

Addressing the above issues for both restricted and unrestricted ATs, this paper theoretically analyzes the impact of the AEs budgets on AT using a simple classification problem, and proposes a novel AT approach named \textbf{blind adversarial training (BAT)} that uses a \textbf{cutoff-scale~(CoS)} strategy to ameliorate the generation of the AEs during training. The steps of the BAT procedure are summarized as follows: after obtaining the AEs by DeepFool at each AE generation step, we \textbf{cutoff} the AEs by an adaptive budget (the mean value of the norm of the current AEs), and then uniformly \textbf{scale} them. The cutoff is used to fix the AEs with large or even unreasonable strength, so as to approach  the perfect decision boundary. The scale is used to prevent the AEs from going over the decision boundary. Both strategies will adaptively estimate a nonuniform budget and ensure that the AEs are located in a reasonable range in blind; this approach tends to obtain a model that is robust overall.

As shown in Figure~\ref{fig:tcc}~(c) for the TCC problem, BAT can generate a model with an exactly regular octagon decision boundary, i.e., the model with the best robustness within the given network architecture, without budget prescribed. Furthermore, compared with the FGSM-, PGD- and DeepFool-based ATs, using BAT to train LeNet-5 on MNIST, Fashion-MNIST and SVHN,  FitNet-4 on CIFAR10 and CIFAR100, we can obtain models with comprehensive robustness, obtaining both high accuracy and robustness for various white/black-box attacks (including the FGSM, Noise, PGD, DeepFool and CW attacks) with varying attack strength. In addition, for MNIST, we clearly find that both the cutoff and scale strategies make significant individual contributions to BAT. 

\section{Methodology}

\subsection{Restricted Adversarial Training}

We consider a standard classification task with dataset $\{\mathbf{x}, \mathbf{y}\}$ and minimizing $\min_\theta \mathbb{E}_{(\mathbf{x},\mathbf{y})} [ \mathcal{L}(\theta,\mathbf{x},\mathbf{y}) ]$, where $\mathcal{L}$ is the loss function with weights $\theta$. Adversarial training (AT) is also called brute-force AT, first proposed by Szegedy \etal.~\cite{szegedy2013intriguing} and further developed by Goodfellow \etal.~\cite{2014arXiv1412.6572G}. The core idea of AT is to enhance the robustness through adding the AEs to the training data; here, the total loss function can be written in a general form $\min_\theta \mathbb{E}_{(\mathbf{x},\mathbf{y})} [\mathcal{L}(\theta,\mathbf{x},\mathbf{y}) +\lambda \mathcal{L}(\theta,\mathbf{x}+\delta(\mathbf{x}),\mathbf{y})]$, where $\mathbf{x}+\delta(\mathbf{x})$ represents the AEs of data $\mathbf{x}$ and is usually set to $\lambda = 1$. AT will alternately generate AEs and optimize network parameters until the levels of accuracy on clean data and AEs converge. To generate AEs, Madry \etal.~\cite{madry2017towards} introduced a set of allowed perturbations $\mathcal{S}$ that formalize the manipulative power of the adversary, usually using the $\ell_\infty$-ball around $\mathbf{x}$ with the budget $\varepsilon$ as $\mathcal{S}$~\cite{2014arXiv1412.6572G}, so that the AT process can be reformulated as $\min_\theta \mathbb{E}_{(\mathbf{x},\mathbf{y})} \max_{\delta(\mathbf{x}) \in \mathcal{S}} \mathcal{L}(\theta, \mathbf{x}+\delta(\mathbf{x}),\mathbf{y})$. This saddle point optimization problem specifies a clear goal of a robust classifier. The inner ``$\max$'' (adversarial loss) aims to find an AE of the given data $\mathbf{x}$, while the outer ``$\min$'' finds a model that minimizes the “adversarial loss”. These AEs can be easily simplified to the widely accepted and used FGSM or PGD~\cite{madry2017towards}. 

Most neural networks are highly non-linear and complex. To theoretically analyze AT, we start from a simplified $1$-layer perceptron defined as, $y=\sigma(W\mathbf{x} + b)$. For classifying two points $\mathbf{x}^1, \mathbf{x}^2$, it is easy to know that the decision boundary of the model with \textbf{best robustness} falls on the perpendicular bisector of the two points. The following proposition can guarantee improved performance of AT (details in Appendix~D.3), 
\begin{prop} For classifying two points $\mathbf{x}^1, \mathbf{x}^2$ with a $1$-layer perceptron model, while both restricted AT and NT can obtain the model with the best robustness, the performance of restricted AT is much better than that of NT because the prescribed budget accelerates the learning process.
\end{prop}

In fact, the best budgets (with maximum acceleration) of different data are different and unpredictable, so that any prescribed uniform budget may lead some of the AEs to be located beyond the decision boundary; thus, the use of a uniform budget is not appropriate. Meanwhile, numerical results show that the robustness is directly related to the budget, so that the choice of  the budget is important and is a difficult problem in restricted AT. 

\subsection{Unrestricted Adversarial Training}

To blind the budget of restricted AT, an intuitive approach is to use the unconstrained AEs (with both the attacking ability and the minimal norm) during AT, formally 
\begin{equation}
\min_\theta \mathbb{E}_{(\mathbf{x},\mathbf{y})} \max_{\delta(\mathbf{x})} \{\mathcal{L}(\theta, \mathbf{x}+\delta(\mathbf{x}),\mathbf{y}) - \lambda||\delta(\mathbf{x})||\}, 
\label{eq:opt_df_at}
\end{equation}
where $\mathcal{L}(\theta, \mathbf{x}+\delta(\mathbf{x}),\mathbf{y})$ term corresponds to guarantee the attacking ability, and $\lambda||\delta(\mathbf{x})||$ term constraint the norm small.
An example of an alternative approach to generate the AEs is DeepFool~\cite{moosavi2016deepfool}, which minimizes the $\ell_2$ norm of AEs slightly beyond the decision boundary of the current model. The AT with unconstrained AEs may obtain a more robust model than NT, but cannot ensure better performance than restricted AT and incur a heavy cost in terms of accuracy loss. Similarly, we state the following proposition (details in Appendix~D.4),
\begin{prop}
For classifying two points $\mathbf{x}^1, \mathbf{x}^2$ with a $1$-layer perceptron model, 
unrestricted AT cannot be used to obtain the model with the best robustness.
\end{prop}

While simply using unconstrained AEs in unrestricted AT avoids the shortcoming of choosing a prescribed budget, the AEs located on the decision boundary ensure that the model with the best robustness cannot be obtained. On the contrary, this will lead to obtaining a model with a more poor decision boundary.


\subsection{Blind Adversarial Training}
\label{sec:cos}

Therefore, to alleviate the drawbacks of the restricted and unrestricted AT, we propose a {cutoff-scale}~(CoS) strategy based on the DeepFool-AT (DF-AT) in \eqref{eq:opt_df_at}, and name our approach blind adversarial training (BAT). We choose DeepFool AEs in AT because the unrestricted AE method DeepFool aims for the decision boundary of a model and obtains AEs that are slightly beyond the boundary with relatively small computational complexity compared to the other similar methods. BAT can be formulated as 
\begin{eqnarray}
\min_\theta \mathbb{E}_{(\mathbf{x},\mathbf{y})} \max_{\delta(\mathbf{x})} \{\mathcal{L}(\theta,  \underbrace{\mathbf{x}+\rho \delta(\mathbf{x})}_{\rm Scale},\mathbf{y}) \nonumber\\
- \lambda_1||\delta(\mathbf{x})|| - \underbrace{\lambda_2(||\delta(\mathbf{x})|| - \varepsilon)_{+}}_{\rm Cutoff}\}, 
\label{eq:opt_bat}
\end{eqnarray}
where coefficient $\lambda_2 \gg \lambda_1$, and we introduce two parameters $\varepsilon$ and $\rho$ to monitor the cutoff and the scale process, respectively. $\mathcal{L}(\theta, \mathbf{x}+\delta(\mathbf{x}),\mathbf{y})$ and $\lambda||\delta(\mathbf{x})||$ terms work the same as unrestricted AT for guarantee the attacking ability and minimal norm.

The motivation of the cutoff and scale is to ensure that the AEs are dynamically located in a reasonable range such that the AT model can be robust when encountering an attack with varying strength. We use the cutoff to further penalize the AEs with a norm larger than the budget $\varepsilon$. Similar to the inner problem in \eqref{eq:opt_df_at}, we can alternatively solve the inner problem in \eqref{eq:opt_bat} followed by the DeepFool AEs. Due to the dramatic fluctuations in the decision boundary during training, the cutoff can avoid the AEs with a large or even unreasonable strength, e.g., some failure or overestimation AEs of DeepFool,  limiting the AEs to lie within the perfect decision boundary. We use the scale to prevent the AEs from going over the decision boundary, i.e., preventing the AEs from different labels from touching each other. 

The procedure of BAT is given in Algorithm \ref{alg:bat}. Starting with the AEs generated by DeepFool which corresponds to only maximizing the first two terms in the objective of \eqref{eq:opt_bat} (with $\rho=1$), we simply calculate the third term (equivalent to minimizing $(||\delta(\mathbf{x})|| - \varepsilon)_{+}$) by cutting off the perturbations of AEs with a norm larger than $\varepsilon$, defined by $\text{cut}\{\delta(\mathbf{x}),\varepsilon\}$
: if $||\delta(\mathbf{x})|| > \varepsilon$, $\delta(\mathbf{x}) \leftarrow \delta(\mathbf{x}) / ||\delta(\mathbf{x})|| \cdot \varepsilon$; otherwise, no change on $\delta(\mathbf{x})$.
Then, we scale the new perturbations with weight $\rho$ and add these CoS AEs into the training process, i.e., $\{\mathbf{x}+\rho \delta(\mathbf{x})\}$. 
We set $\varepsilon= \mathbb{E}||\delta(\mathbf{x})||$ and $\rho<1$ (a predefined parameter), corresponding to adaptively estimating a nonuniform budget. 
$\text{cut}\{\delta(\mathbf{x}),\mathbb{E}||\delta(\mathbf{x})||\}$ implies that the budget $\varepsilon\leftarrow \mathbb{E}||\delta(\mathbf{x})||$ is computed prior to cutting off the perturbations $\delta(\mathbf{x})$.  

Similar to the restricted and unrestricted AT, it can be proved theoretically for BAT that (details in Appendix~D.5),
\begin{prop}
For classifying two points $\mathbf{x}^1, \mathbf{x}^2$ with a $1$-layer perceptron model, 
BAT also can obtain the model with the best robustness, and has the same convergence property as restricted AT with best budget.
\end{prop} 

Therefore, BAT not only avoids the difficulty of making the best choice of the budget, but also provides an AT approach that can dynamically adjust a nonuniform budget, seek to provide a path to potentially guess the perfect decision boundary, and finally reach the model with best robustness. Therefore, we name this approach blind adversarial training.

\begin{algorithm}[tb]
  \caption{Blind Adversarial Training (BAT)}
  \label{alg:bat}
  \begin{algorithmic}
    \REQUIRE {Dataset $\{\mathbf{x},\mathbf{y}\}$ and hyper-parameters~(scale factor $\rho$, learning rate $\alpha$).}
    \ENSURE {Model $\theta$.}
    \STATE {Initialize model $\theta$.}
    \REPEAT
    \STATE {$\mathcal{L}_{\rm C} = \mathcal{L}(\theta,\mathbf{x},\mathbf{y})$,} \COMMENT{Loss on clean data}
    \STATE {$\delta(\mathbf{x}) = \mathbf{x}_{\rm adv} - \mathbf{x}$,} \COMMENT{$\mathbf{x}_{\rm adv}$: DeepFool AEs}
    \STATE {$\varepsilon = \mathbb{E}_{\mathbf{x}}\,||\delta(\mathbf{x})||_2$,} \COMMENT{Adaptive Cutoff budget}
    \STATE {$\delta_{\rm Co}(\mathbf{x}) = \text{cut}\{\delta(\mathbf{x}), \varepsilon\}$,} \COMMENT{Cutoff AEs}
    \STATE {$\delta_{\rm CoS}(\mathbf{x}) = \rho \delta_{\rm Co}(\mathbf{x})$,} \COMMENT{Scale AEs}
    \STATE {$\mathbf{x}_{\rm CoS} = \mathbf{x} + \delta_{\rm CoS}(\mathbf{x})$,} \COMMENT{Get CoS AEs}
    \STATE {$\mathcal{L}_{\rm AE} = \mathcal{L}(\theta,\mathbf{x}_{\rm CoS},\mathbf{y}) $,} \COMMENT{Loss on CoS AEs}
    \STATE {$\theta = \theta - \alpha (\nabla_{\theta} \mathcal{L}_{\rm C} + \nabla_{\theta} \mathcal{L}_{\rm AE})$,} 
        \COMMENT{Update with total loss}
    \UNTIL{$\text{avg-}\A\A(1)$ converge.} \COMMENT{comprehensive robustness}
  \end{algorithmic}
\end{algorithm}

\textbf{Adversarial Accuracy}. 
To evaluate the performance of AT, some approaches use adversarial perturbations or their average as the measurement~\cite{cw,7467366}; however, this makes AT sensitive to the AEs with a large norm. Other approaches instead use the accuracy under adversarial attacks with the given budget~\cite{2014arXiv1412.6572G,madry2017towards}. However, a fixed attack budget cannot generally demonstrate the full robustness of a model, since any model can be attacked in the absence of restrictions on the attack budget. 
Correspondingly, we propose a new evaluation criterion, named adversarial accuracy ($\A\A$), to monitor the training process. $\A\A$ is defined as the accuracy curve for white-box attacks with varying attack strength: 
\begin{equation}\label{eq:aa}
\A\A(\varepsilon) = \A(\mathbf{x} + \text{cut} \{\delta(\mathbf{x}), \varepsilon \}  )
\end{equation}
where $\A( \cdot )$ is the accuracy of the given data. 
$\A\A$ can be considered as an evaluation of comprehensive robustness.
To simplify the expression, \textbf{average adversarial accuracy ($\text{avg-}\A\A$)} is defined to show the overall robustness in an interval $[0,\Theta]$, 
\begin{equation}\label{eq:aaa}
\text{avg-}\A\A(\Theta)=\frac{1}{\Theta} \int_{0}^{\Theta} \A\A(\varepsilon) d \varepsilon.
\end{equation}
A larger $\text{avg-}\A\A$ means that the model shows more robustness under the attacks with maximum budget $\Theta$.
Clearly, clean data accuracy is a special case of $\A\A$ and $\text{avg-}\A\A$ with $0$ attack strength,
$\text{avg-}\A\A(0) = \lim_{\Theta \rightarrow 0} \frac{1}{\Theta} \int_{0}^{\Theta} \A\A(\varepsilon) d \varepsilon = \A\A(0) = \A(\mathbf{x})$.

\section{Experiments}
\label{sec:result}

In this section, we evaluate the BAT approach on various benchmark datasets, facing FGSM~\cite{2014arXiv1412.6572G}, Noise~\cite{goodfellow2019evaluation}, PGD~\cite{madry2017towards}, DeepFool~\cite{moosavi2016deepfool} and CW~\cite{cw} attacks. 
The code for these experiments is based on the open source library cleverhans~\cite{papernot2018cleverhans}.
We consider using BAT to train LeNet-$5$~\cite{MnistLeNet} for MNIST~\cite{MnistLeNet}, Fashion-MNIST~\cite{xiao2017/online} and SVHN~\cite{Netzer2012Reading}, and train FitNet-$4$~\cite{Romero2015} for CIFAR-10 and CIFAR-100~\cite{Krizhevsky2009}. We compare the BAT approach with several state-of-the-art training approaches, such as normal training (NT), FGSM-AT (AT with FGSM AEs), PGD-AT (AT with PGD AEs), and DF-AT (AT with DeepFool AEs). 

{\bf Experimental setup}. For all of the experiments, we normalize the pixels to $[0,1]$ by dividing by 255, use label smoothing regularization~\cite{7780677} to avoid over-fitting, and perform data augmentation (with a width/height shift range of $0.1$ and random flips) for the SVHN, CIFAR-10 and CIFAR-100 datasets to improve the clean data accuracy. For the inner AE generations for all of the datasets, see Table~\ref{Parameters_AE} for details. 
For the outer training process in AT, we use Adam optimization~\cite{Kingma2014Adam} and see Table~\ref{Parameters_AT} for details.
We set the scale parameter $\rho = 0.9$ in BAT for the final results. 

{\bf Compute Overhead}. As shown in Algorithm~\ref{alg:bat}, the main add-ons of BAT is the cutoff\&scale strategy, which is quite simple from the point of view of computation cost. 
Thus the extra expense (compute/memory overhead) of BAT than DF-AT is negligible.

\begin{table}[t]
\caption{Parameters of AEs}
\label{Parameters_AE}
\centering
{\begin{tabular}{|c|c|c|}
\hline
AEs       & Parameter & Value \\
\hline 
\multirow{2}*{FGSM} & Perturbation & Norm-constrained \\
                    & Types of norm & $\ell_\infty$ \\
                    \hline
\multirow{2}*{Noise} & Perturbation & Norm-constrained \\
                    & Types of norm & $\ell_\infty$ \\
                    \hline
\multirow{5}*{PGD}  & Perturbation & Norm-constrained \\
                    & Types of norm & $\ell_\infty$ \\
                    & Number of step & $20$ \\
                    & Perturbation per step & $\varepsilon/10$ \\
                    & Initial perturbation & $\varepsilon/2$ \\
                    \hline
\multirow{4}*{DeepFool}  & Perturbation & Unconstrained \\
                    & Types of norm & $\ell_2$ \\
                    & Number of step & $10$ \\
                    & Overshoot & $0.02$ \\
                    \hline
\multirow{6}*{CW}  & Perturbation & Unconstrained \\
                    & Types of norm & $\ell_2$ \\
                    & Number of step & $100$ with abort early\\
                    & Learning rate & $0.01$ \\
                    & Confidence & $0$ \\
                    & Binary search steps & $10$\\
\hline
\end{tabular}}
\end{table}

\begin{table}[t]
\caption{Parameters of AT}
\label{Parameters_AT}
\centering
\setlength{\tabcolsep}{2.0mm}{
\begin{tabular}{|c|c|c|}
\hline 
Dataset    & $N_\text{epochs}$  &  Learning rate  \\
\hline 
MNIST      & 50          & 1e-3               \\
Fashion-MNIST & 70       & 1e-3               \\
SVHN       & 300         & 1e-4               \\
CIFAR-10   & 60          & 1e-4               \\
CIFAR-100  & 100         & 1e-4            \\
\hline 
\end{tabular}}
\end{table}

\begin{figure*}[!t]
  \centering
  \includegraphics[width=0.9\textwidth]{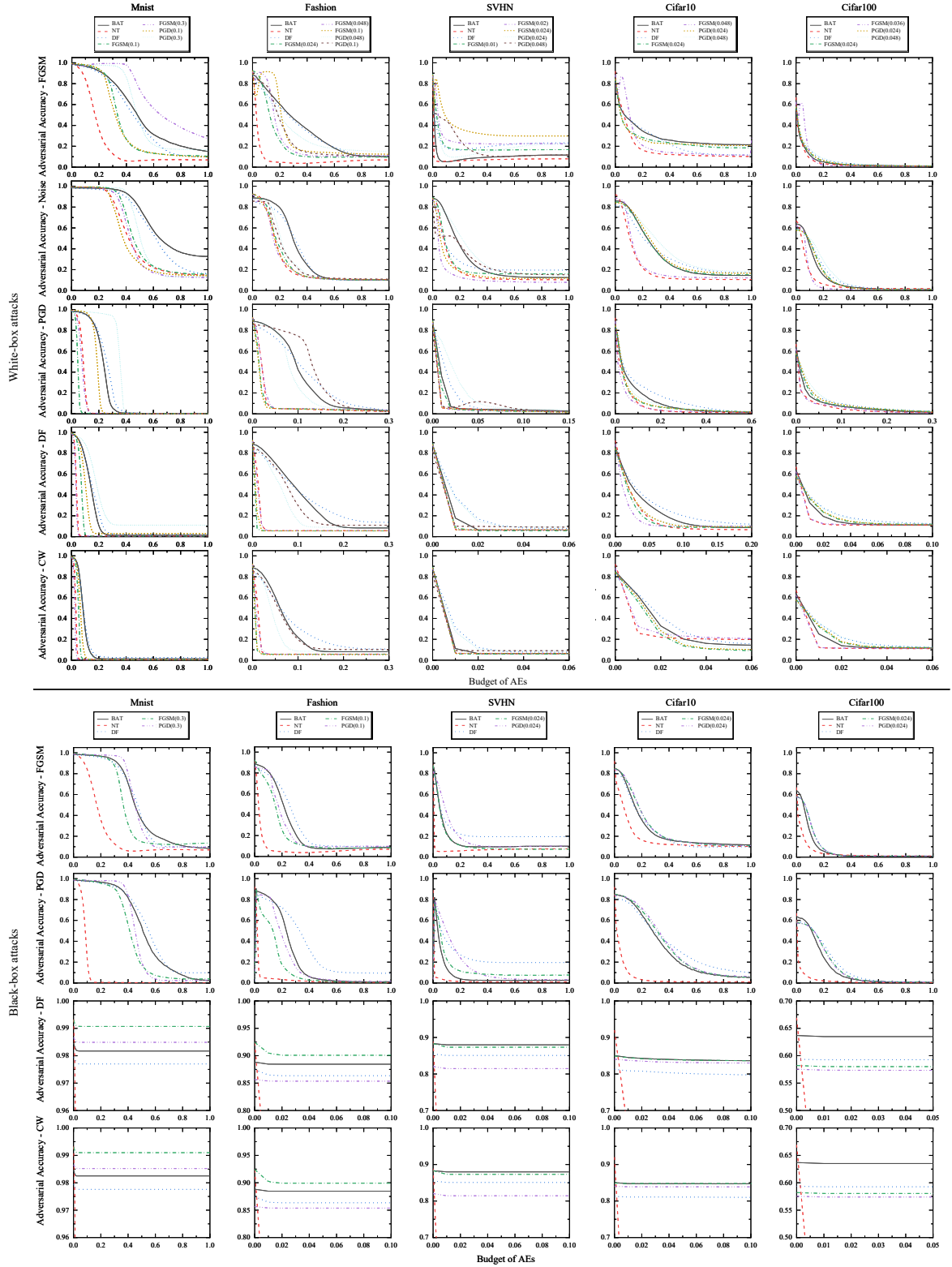}
  \caption{Comparison of the $\A\A$ of BAT with NT, DF-AT, FGSM-AT and PGD-AT approaches for various datasets (MNIST, Fashion-MNIST, SVHN, CIFAR10 and CIFAR100) and various white/black-box attacks (FGSM, Noise, PGD, DeepFool and CW attacks) with varying attack strengths. 
  The correspondence between the approaches (with the budget in brackets) and color/shape of the lines are shown in the above legend. }
  \label{fig:w-bat}
\end{figure*}

\begin{table*}[!t]

  \caption{Comparison of the accuracy on clean data/AEs with varying attack strengths and avg-$\A\A(\Theta)$ (overall robustness, Eq.~\eqref{eq:aaa}) of BAT with other 
  approaches (with the prescribed budget) for various datasets 
  and various $\ell_2$ white/black-box attacks, 
  The larger the average of AEs norm ($\mathbb{E}||\delta(\mathbf{x})||$, denoted by "Norm" in Table) corresponds to the robustness better. For black-box attacks, the mean value is that of NT. 
  We mark the highest result with red bold and the second highest result with blue underline. 
  }
\label{compare}
\vspace{0.5cm}
\centering
\footnotesize
\setlength{\tabcolsep}{2.0mm}{
\begin{tabular}{|c|c|c|c|cc|cccc|cccc|}
\hline
 &
 & Defense     & Clean   & \multicolumn{2}{c|}{avg-$\A\A(\Theta)$}
           & \multicolumn{4}{c|}{DeepFool ($\ell_2$)}
           & \multicolumn{4}{c|}{CW ($\ell_2$)}\\
\hline 
\multirow{30}*{\rotatebox{90}{White-box}} & \multirow{6}*{\rotatebox{90}{Mnist}} & Budget       & $-$      & 0.13 & 0.07
             &0.05&0.1&0.2& Norm 
             &0.03&0.05&0.1&Norm\\
\cline{3-14}
& &NT &\textbf{\textcolor{red}{99.3}}   & 28.7& 49.8
    &4.2 &0.6 &0.6 &0.035 
    &58.4 &0.7 &0.6 &0.03 \\
& &FGSM(0.3)&\textbf{\textcolor{red}{99.3}}   & 26.0& 45.0
    &0.8 &0.6 &0.6 &0.031 
    &14.2 &0.6 &0.6 &0.023 \\
& &PGD(0.3) &\underline{\textcolor{blue}{98.7}}   & \textbf{\textcolor{red}{92.2}}& \textbf{\textcolor{red}{96.3}}
    &\textbf{\textcolor{red}{95.4}} &\textbf{\textcolor{red}{88.9}} &\textbf{\textcolor{red}{40.1}} &\textbf{\textcolor{red}{0.182}} 
    &94.8 &83.7 &29.1 &\underline{\textcolor{blue}{0.083}} \\
& &DF-AT &97.9   & \underline{\textcolor{blue}{84.8}} & 94.2
    &92.8 &\underline{\textcolor{blue}{75.3}} &\underline{\textcolor{blue}{23.6}} &\underline{\textcolor{blue}{0.147}} 
    &\underline{\textcolor{blue}{95.3}} &\underline{\textcolor{blue}{87.2}} &\textbf{\textcolor{red}{35.6}} &\textbf{\textcolor{red}{0.087}} \\
& &BAT &98.5   & \underline{\textcolor{blue}{84.8}}& \underline{\textcolor{blue}{94.8}}
    &\underline{\textcolor{blue}{93.4}} &{{}{75.0}} &10.7 &0.136 
    &\textbf{\textcolor{red}{96.3}} &\textbf{\textcolor{red}{88.4}} &\underline{\textcolor{blue}{29.6}} &0.08 \\
\cline{2-14}

& \multirow{6}*{\rotatebox{90}{Fashion}} & Budget       & $-$        & 0.08& 0.04
             &0.05&0.1&0.2&Norm
             &0.03&0.05&0.1&Norm\\
\cline{3-14}
& &NT &\underline{\textcolor{blue}{92.3}}    & 22.4 & 35.6		
    &5.7 &5.7 &5.7 &0.011 
    &5.8 &5.8 &5.8 &0.01  \\
& &FGSM(0.1)&\textbf{\textcolor{red}{92.5}}   & 17.2  & 26.5		
    &5.6 &5.6 &5.6&0.007  
    &5.5 &5.5 &5.5 &0.005  \\
& &PGD(0.1) &85.7   & 69.3  & \underline{\textcolor{blue}{79.6}}	
    &\underline{\textcolor{blue}{66.6}} &29.9 &\underline{\textcolor{blue}{11.0}} &0.065 
    &\underline{\textcolor{blue}{72.5}} &55.7 &19.7 &\underline{\textcolor{blue}{0.051}}  \\
& &DF-AT &87.5   & \underline{\textcolor{blue}{69.4}}& 77.8	
    &65.1 &\underline{\textcolor{blue}{43.8}} &\textbf{\textcolor{red}{18.6}} &\underline{\textcolor{blue}{0.07}} 
    &70.7 &\underline{\textcolor{blue}{57.9}} &\textbf{\textcolor{red}{31.3}} &0.046  \\
& &BAT &88.8   & \textbf{\textcolor{red}{73.9}}& \textbf{\textcolor{red}{83.2}}
    &\textbf{\textcolor{red}{70.3}} &\textbf{\textcolor{red}{44.2}} &9.1 &\textbf{\textcolor{red}{0.081}} 
    &\textbf{\textcolor{red}{77.5}} &\textbf{\textcolor{red}{60.7}} &\underline{\textcolor{blue}{21.9}} &\textbf{\textcolor{red}{0.057}}  \\
\cline{2-14}

& \multirow{6}*{\rotatebox{90}{SVHN}} & Budget       & $-$      & 0.004 &  0.002
             &0.001&0.005&0.01&Norm
             &0.001&0.005&0.01&Norm\\
\cline{3-14}
& &NT &\textbf{\textcolor{red}{88.9}}   & 44.6 &65.4
    &69.0 &6.9 &6.3 &0.002 
    &70.7 &6.6 &6.4 &0.001 \\
& &FGSM(0.024) &\underline{\textcolor{blue}{88.3}}   & 43.0  &63.5
    &66.9 &7.1 &6.6 &0.002 
    &69.3 &6.7 &6.6 &0.001 \\
& &PGD(0.024) &82.0   & 75.8 & 79.8 
    &80.4 &\textbf{\textcolor{red}{62.6}} &\underline{\textcolor{blue}{37.0}} &\underline{\textcolor{blue}{0.008}} 
    &80.4 &\textbf{\textcolor{red}{57.8}} &\underline{\textcolor{blue}{27.6}} &\underline{\textcolor{blue}{0.006}} \\
& &DF-AT &85.5   & \textbf{\textcolor{red}{77.7}} & \underline{\textcolor{blue}{82.5}}	
    &\underline{\textcolor{blue}{83.4}} &\underline{\textcolor{blue}{61.7}} &\textbf{\textcolor{red}{38.4}} &\textbf{\textcolor{red}{0.009}} 
    &\underline{\textcolor{blue}{83.3}} &\underline{\textcolor{blue}{57.6}} &\textbf{\textcolor{red}{31.3}} &\textbf{\textcolor{red}{0.007}} \\
& &BAT &\underline{\textcolor{blue}{88.3}}   & \underline{\textcolor{blue}{76.2}} & \textbf{\textcolor{red}{83.8}}
    &\textbf{\textcolor{red}{85.0}} &50.7 &17.8 &0.005 
    &\textbf{\textcolor{red}{84.8}} &43.6 &11.4 &0.004 \\
\cline{2-14}

& \multirow{6}*{\rotatebox{90}{Cifar10}} & Budget       & $-$       &0.03  & 0.02  
             &0.003&0.01&0.03&Norm
             &0.003&0.01&0.03&Norm\\
\cline{3-14}
& &NT &\textbf{\textcolor{red}{92.0}}   & 58.9 & 67.5		
    &77.2 &64.0 &33.2 &0.029 
    &60.0 &33.2 &20.7 &0.0055 \\
& &FGSM(0.024) &84.9  &54.3  &63.2		
    &81.2 &63.2 &28.0 &0.018 
    &80.9 &56.0 &13.5 &0.012 \\
& &PGD(0.024) &84.0  & 57.3  & 65.2	
    &\underline{\textcolor{blue}{81.5}} &65.5 &33.7 &0.022 
    &\underline{\textcolor{blue}{81.3}} &59.3 &14.7 &0.013 \\
& &DF-AT &81.3   & \textbf{\textcolor{red}{62.9}}& \underline{\textcolor{blue}{68.3}} 
    &78.9 &\underline{\textcolor{blue}{68.3}} &\textbf{\textcolor{red}{46.8}} &\textbf{\textcolor{red}{0.037}} 
    &78.8 &\textbf{\textcolor{red}{63.5}} &\textbf{\textcolor{red}{26.8}} &\textbf{\textcolor{red}{0.016}} \\
& &BAT &\underline{\textcolor{blue}{85.8}}   & \underline{\textcolor{blue}{62.3}}& \textbf{\textcolor{red}{68.8}} 
    &\textbf{\textcolor{red}{82.0}} &\textbf{\textcolor{red}{68.8}} &\underline{\textcolor{blue}{42.8}} &\underline{\textcolor{blue}{0.032}} 
    &\textbf{\textcolor{red}{81.8}} &\underline{\textcolor{blue}{62.0}} &\underline{\textcolor{blue}{20.3}} &\underline{\textcolor{blue}{0.015}} \\
\cline{2-14}

& \multirow{6}*{\rotatebox{90}{Cifar100}} & Budget       & $-$        & 0.007 & 0.003 
             &0.003&0.01&0.03&Norm
             &0.003&0.01&0.03&Norm\\
\cline{3-14}
& &NT &\textbf{\textcolor{red}{66.9}}   & 34.6 & 49.6\
    &35.4 &17.1 &11.3 &0.003 
    &20.4 &11.9 &11.2 &0.002 \\
& &FGSM(0.024) &58.2   &51.2  & 55.9\
    &53.0 &36.3 &17.0 &0.008 
    &52.4 &32.1 &\underline{\textcolor{blue}{13.3}} &0.006 \\
& &PGD(0.024) &57.5   &51.4  & 55.6\
    &52.9 &\underline{\textcolor{blue}{37.4}} &\underline{\textcolor{blue}{17.9}} &\underline{\textcolor{blue}{0.009}} 
    &52.2  &\underline{\textcolor{blue}{33.4}} &13.2 &\underline{\textcolor{blue}{0.007}} \\
& &DF-AT &59.4   & \underline{\textcolor{blue}{52.2}}& \underline{\textcolor{blue}{57.0}}
    &\underline{\textcolor{blue}{54.1}} &\textbf{\textcolor{red}{38.9}} &\textbf{\textcolor{red}{21.5}} &\textbf{\textcolor{red}{0.012}} 
    &\underline{\textcolor{blue}{53.2}} &\textbf{\textcolor{red}{35.2}} &\textbf{\textcolor{red}{16.3}} &\textbf{\textcolor{red}{0.008}} \\
& &BAT &\underline{\textcolor{blue}{63.7}}   &\textbf{\textcolor{red}{52.8}} & \textbf{\textcolor{red}{59.7}}
    &\textbf{\textcolor{red}{54.5}} &32.5 &14.5 &0.007 
    &\textbf{\textcolor{red}{53.4}} &25.1 &12.1 &0.005 \\
\hline 
\multirow{25}*{\rotatebox{90}{Black-box}} &
\multirow{5}*{\rotatebox{90}{Mnist}} & Budget     &$-$    & &
             &0.01&0.03&0.05&Norm
             &0.01&0.03&0.05&Norm \\
\cline{3-14}
& &FGSM(0.3) &\textbf{\textcolor{red}{99.3}}   & &
    &\textbf{\textcolor{red}{99.2}} &\textbf{\textcolor{red}{99.1}} &\textbf{\textcolor{red}{99.1}} & 0.035
    &\textbf{\textcolor{red}{99.2}} &\textbf{\textcolor{red}{99.1}} &\textbf{\textcolor{red}{99.1}} &0.03\\
& &PGD(0.3) &\underline{\textcolor{blue}{98.7}}   & &
    &\underline{\textcolor{blue}{98.6}} &\underline{\textcolor{blue}{98.5}} &\underline{\textcolor{blue}{98.5}} & 0.035
    &\underline{\textcolor{blue}{98.6}} &\underline{\textcolor{blue}{98.5}} &\underline{\textcolor{blue}{98.5}}  & 0.03\\
& &DF-AT &97.9   & &
    &97.8 &97.7 &97.7 &0.035
    &97.8 &97.8 &97.8 & 0.03\\
& &BAT &98.5   & &
    &98.3 &98.2 &98.2 &0.035
    &98.3 &98.3 &98.3 &0.03\\
\cline{2-14}

& \multirow{5}*{\rotatebox{90}{Fashion}} & Budget    &$-$     & &
             &0.01&0.03&0.05 &Norm
             &0.01&0.03&0.05 &Norm\\
\cline{3-14}
& &FGSM(0.1) &\textbf{\textcolor{red}{92.5}}   & &
    &\textbf{\textcolor{red}{90.5}} &\textbf{\textcolor{red}{90.0}} &\textbf{\textcolor{red}{90.0}}  & 0.011
    &\textbf{\textcolor{red}{90.2}} &\textbf{\textcolor{red}{89.9}} &\textbf{\textcolor{red}{89.9}}   & 0.01 \\
& &PGD(0.1)  &85.7   & &
    &85.4 &85.4 &85.4   & 0.011
    &85.4 &85.4 &85.4   &  0.01\\
& &DF-AT  &87.5   & &
    &86.5 &86.4 &86.4  & 0.011
    &86.5 &86.4 &86.4   & 0.01 \\
& &BAT &\underline{\textcolor{blue}{88.8}}   & &
    &\underline{\textcolor{blue}{88.5}} &\underline{\textcolor{blue}{88.5}} &\underline{\textcolor{blue}{88.5}}  & 0.011
    &\underline{\textcolor{blue}{88.5}} &\underline{\textcolor{blue}{88.5}} &\underline{\textcolor{blue}{88.5}}   & 0.01 \\
\cline{2-14}

& \multirow{5}*{\rotatebox{90}{SVHN}} & Budget     &$-$    & &  
             &0.001&0.002&0.003 &Norm 
             &0.001&0.002&0.003 &Norm \\
\cline{3-14}
& &FGSM(0.024)  &\textbf{\textcolor{red}{88.3}}   & &
    &\underline{\textcolor{blue}{87.3}} &\underline{\textcolor{blue}{87.3}} &\underline{\textcolor{blue}{87.3}}  & 0.002
    &\underline{\textcolor{blue}{87.3}} &\underline{\textcolor{blue}{87.3}} &\underline{\textcolor{blue}{87.3}}  & 0.001 \\
& &PGD(0.024)  &82.0   & &
    &81.5 &81.5 &81.5  & 0.002
    &81.4 &81.4 &81.4  & 0.001 \\
& &DF-AT  &\underline{\textcolor{blue}{85.5}}   & &
    &85.1 &85.1 &85.1   & 0.002
    &85.1 &85.1 &85.1   & 0.001\\
& &BAT &\textbf{\textcolor{red}{88.3}}   & &
    &\textbf{\textcolor{red}{88.0}} &\textbf{\textcolor{red}{88.0}} &\textbf{\textcolor{red}{88.0}}  &  0.002
    &\textbf{\textcolor{red}{88.0}} &\textbf{\textcolor{red}{88.0}} &\textbf{\textcolor{red}{88.0}}  &  0.001\\
\cline{2-14}

& \multirow{5}*{\rotatebox{90}{Cifar10}} & Budget     &$-$      & &
             &0.001&0.002&0.003 &Norm 
             &0.001&0.002&0.003 &Norm \\
\cline{3-14}
& &FGSM(0.024)  &\underline{\textcolor{blue}{84.9}}   & &
    &\textbf{\textcolor{red}{84.5}} &\underline{\textcolor{blue}{84.2}} &\underline{\textcolor{blue}{84.0}}  & 0.029
    &\textbf{\textcolor{red}{84.8}} &\textbf{\textcolor{red}{84.8}} &\textbf{\textcolor{red}{84.8}}  & 0.0055 \\
& &PGD(0.024)  &84.0   & &
    &83.7 &83.5 &83.3  & 0.029
    &\underline{\textcolor{blue}{83.9}} &\underline{\textcolor{blue}{83.9}} &\underline{\textcolor{blue}{83.9}}  & 0.0055 \\
& &DF-AT  &81.3    & &
    &80.9 &{80.8} &{80.5}  &  0.029
    &81.1 &{81.1} &81.1   & 0.0055\\
& &BAT &\textbf{\textcolor{red}{85.8}}    & &
    &\textbf{\textcolor{red}{84.5}} &\textbf{\textcolor{red}{84.3}} &\textbf{\textcolor{red}{84.1}}  & 0.029
    &\textbf{\textcolor{red}{84.8}} &\textbf{\textcolor{red}{84.8}} &\textbf{\textcolor{red}{84.8}}   & 0.0055\\
\cline{2-14}

& \multirow{5}*{\rotatebox{90}{Cifar100}} & Budget      &$-$     & &
             &0.001&0.002&0.003 &Norm 
             &0.001&0.002&0.003 &Norm \\
\cline{3-14}
& &FGSM(0.024)  &58.2   & &
    &58.0 &58.0 &58.0  & 0.003
    &58.1 &58.1 &58.1   & 0.002\\
& &PGD(0.024)  &57.5   & &
    &57.4 &57.4 &57.4  & 0.003
    &57.4 &57.4 &57.4   & 0.002\\
& &DF-AT  &\underline{\textcolor{blue}{59.4}}   & &
    &\underline{\textcolor{blue}{59.3}} &\underline{\textcolor{blue}{59.3}} &\underline{\textcolor{blue}{59.3}}  & 0.003
    &\underline{\textcolor{blue}{59.3}} &\underline{\textcolor{blue}{59.3}} &\underline{\textcolor{blue}{59.3}}  &  0.002\\
& &BAT &\textbf{\textcolor{red}{63.7}}   & &
    &\textbf{\textcolor{red}{63.5}} &\textbf{\textcolor{red}{63.5}} &\textbf{\textcolor{red}{63.5}}  & 0.003
    &\textbf{\textcolor{red}{63.5}} &\textbf{\textcolor{red}{63.5}} &\textbf{\textcolor{red}{63.5}}   & 0.002\\
\hline
\end{tabular}
}
\end{table*}

\begin{table*}[!t]
  \caption{Comparison of the accuracy on clean data and AEs with varying attack strengths (small, medium and large) of BAT with other 
  approaches (with the prescribed budget in brackets) for various datasets 
  and various $\ell_\infty$ white/black-box attacks, 
  the Noise attack in the white-box is also a black-box attack. 
  We mark the highest result with red bold and the second highest result with blue underline.
  Notice that BAT has zero knowledge about the attack budgets, 
  even so it still performs well in most cases.}
\label{compare1}
\vspace{0.5cm}
\centering
\footnotesize
\setlength{\tabcolsep}{2.0mm}{
\begin{tabular}{|c|c|c|c|ccc|ccc|ccc|}
\hline
 &
 & Defense     & Clean   & \multicolumn{3}{c|}{FGSM ($\ell_\infty$)}
           & \multicolumn{3}{c|}{Noise ($\ell_\infty$)}
           & \multicolumn{3}{c|}{PGD ($\ell_\infty$)}\\
\hline 
\multirow{30}*{\rotatebox{90}{White-box}} & \multirow{6}*{\rotatebox{90}{Mnist}} & Budget       & $-$      &0.1&0.3&0.5
             &0.1&0.3&0.5
             &0.1&0.2&0.3\\
\cline{3-13}
& &NT &\textbf{\textcolor{red}{99.3}} &84.4 &12.8 &6.3 
    &\textbf{\textcolor{red}{99.3}} &87.7 &31.7 
    &23.8 &0.5 &0.5 \\
& &FGSM(0.3)&\textbf{\textcolor{red}{99.3}} &\underline{\textcolor{blue}{97.0}} &\textbf{\textcolor{red}{99.3}} &\textbf{\textcolor{red}{73.8}} 
    &99.1 &94.8 &27.4 
    &15.9 &0.6 &0.6 \\
& &PGD(0.3) &\underline{\textcolor{blue}{98.7}} &\textbf{\textcolor{red}{98.3}} &\underline{\textcolor{blue}{97.3}} &42.3 
    &\underline{\textcolor{blue}{98.7}} &\textbf{\textcolor{red}{98.5}} &51.9 
    &\textbf{\textcolor{red}{98.0}} &\textbf{\textcolor{red}{96.6}} &\textbf{\textcolor{red}{93.8}}  \\
& &DF-AT &97.9 &95.9 &79.3 &40.9 
    &97.7 &95.3 &\underline{\textcolor{blue}{73.6}} 
    &95.1 &\underline{\textcolor{blue}{77.4}} &\underline{\textcolor{blue}{	22.7}}	 \\
& &BAT &98.5 &96.7 &82.2 &\underline{\textcolor{blue}{45.8}} 
    &98.3 &\underline{\textcolor{blue}{97.3}} &\textbf{\textcolor{red}{78.9}} 
    &\underline{\textcolor{blue}{95.6}} &72.6 &5.8  \\
\cline{2-13}

& \multirow{6}*{\rotatebox{90}{Fashion}} & Budget       & $-$      &0.05&0.1&0.2
             &0.05&0.1&0.2
             &0.05&0.1&0.2\\
\cline{3-13}
& &NT &\underline{\textcolor{blue}{92.3}} &21.6 &6.7 &4.6 
    &\textbf{\textcolor{red}{90.0}} &78.5 &35.0 
    &5.1 &4.5 &3.3 \\
& &FGSM(0.1)&\textbf{\textcolor{red}{92.5}} &\textbf{\textcolor{red}{83.9}} &\textbf{\textcolor{red}{91.5}} &\underline{\textcolor{blue}{59.0}} 
    &87.5 &77.4 &34.3 
    &5.1 &4.8 &3.8 \\
& &PGD(0.1) &85.7 &82.9 &\underline{\textcolor{blue}{78.5}} &45.1 
    &85.4 &\underline{\textcolor{blue}{85.1}} &48.1 
    &\textbf{\textcolor{red}{81.7}} &\textbf{\textcolor{red}{74.6}} &\underline{\textcolor{blue}{7.5}}  \\
& &DF-AT &87.5 &77.7 &70.1 &57.4 
    &83.6 &81.5 &\underline{\textcolor{blue}{73.7}} 
    &72.9 &\underline{\textcolor{blue}{49.0}} &\textbf{\textcolor{red}{11.3}} \\
& &BAT &88.8 &\underline{\textcolor{blue}{83.1}} &75.0 &\textbf{\textcolor{red}{60.3}} 
    &\underline{\textcolor{blue}{88.4}} &\textbf{\textcolor{red}{87.4}} &\textbf{\textcolor{red}{79.4}} 
    &\underline{\textcolor{blue}{79.4}} &42.3 &5.8 \\
\cline{2-13}

& \multirow{6}*{\rotatebox{90}{SVHN}} & Budget       & $-$      &0.05&0.1&0.2
             &0.05&0.1&0.2
             &0.001&0.005&0.01\\
\cline{3-13}
& &NT &\textbf{\textcolor{red}{88.9}} &5.4 &5.6 &6.0 
    &69.8 &39.2 &18.2 
    &78.9 &12.5 &5.6 \\
& &FGSM(0.024) &\underline{\textcolor{blue}{88.3}} &\textbf{\textcolor{red}{68.3}} &\textbf{\textcolor{red}{47.5}} &\textbf{\textcolor{red}{36.8}} 
    &56.4 &29.8 &18.4 
    &77.4 &11.2 &5.7 \\
& &PGD(0.024) &82.0 &26.7 &11.6 &10.1 
    &\underline{\textcolor{blue}{79.2}} &\textbf{\textcolor{red}{73.5}} &\textbf{\textcolor{red}{52.8}} 
    &81.8 &\textbf{\textcolor{red}{78.1}} &\textbf{\textcolor{red}{69.5}} \\
& &DF-AT &85.5 &\underline{\textcolor{blue}{27.3}} &\underline{\textcolor{blue}{18.8}} &\underline{\textcolor{blue}{16.4}} 
    &63.9 &43.4 &27.7 
    &\underline{\textcolor{blue}{84.7}} &\underline{\textcolor{blue}{73.8}} &\underline{\textcolor{blue}{54.9}} \\
& &BAT &\underline{\textcolor{blue}{88.3}} &7.2 &5.1 &7.0 
    &\textbf{\textcolor{red}{84.5}} &\underline{\textcolor{blue}{72.0}} &\underline{\textcolor{blue}{43.2}} 
    &\textbf{\textcolor{red}{86.9}} &67.8 &35.5 \\
\cline{2-13}

& \multirow{6}*{\rotatebox{90}{Cifar10}} & Budget       & $-$      &0.05&0.1&0.2
             &0.05&0.1&0.2
             &0.003&0.01&0.03\\
\cline{3-13}
& &NT &\textbf{\textcolor{red}{92.0}} &\textbf{\textcolor{red}{56.0}} &35.0 &16.4 
    &80.0 &53.8 &21.6 
    &76.4 &58.7 &39.5 \\
& &FGSM(0.024) &84.9&45.4  &35.3 &28.3 
    &\underline{\textcolor{blue}{83.5}} &77.7 &\underline{\textcolor{blue}{58.1}} 
    &83.3 &\underline{\textcolor{blue}{74.1}} &41.9 \\
& &PGD(0.024) &84.0&45.9  &34.6 &26.5 
    &83.3 &\textbf{\textcolor{red}{79.6}} &\textbf{\textcolor{red}{61.3}} 
    &\textbf{\textcolor{red}{83.7}} &\textbf{\textcolor{red}{76.6}} &46.6 \\
& &DF-AT &81.3 &\underline{\textcolor{blue}{54.1}} &\textbf{\textcolor{red}{46.3}} &\textbf{\textcolor{red}{36.7}} 
    &79.9 &72.5 &51.1 
    &80.8 &72.2 &\textbf{\textcolor{red}{50.1}} \\
& &BAT &\underline{\textcolor{blue}{85.8}} &{53.7} &\underline{\textcolor{blue}{45.9}} &\underline{\textcolor{blue}{34.6}} 
    &\textbf{\textcolor{red}{85.0}} &\underline{\textcolor{blue}{79.1}} &57.2 
    &\underline{\textcolor{blue}{83.4}} &73.8 &\underline{\textcolor{blue}{46.8}} \\
\cline{2-13}

& \multirow{6}*{\rotatebox{90}{Cifar100}} & Budget       & $-$      &0.01&0.03&0.05
             &0.05&0.1&0.2
             &0.003&0.01&0.03\\
\cline{3-13}
& &NT &\textbf{\textcolor{red}{66.9}} &39.0 &29.3 &20.1 
    &43.6 &18.3 &6.1 
    &39.3 &22.5 &13.1 \\
& &FGSM(0.024) &58.2 &\underline{\textcolor{blue}{47.5}} &30.8 &23.1 
    &56.3 &\textbf{\textcolor{red}{47.8}} &\underline{\textcolor{blue}{19.2}} 
    &56.0 &\underline{\textcolor{blue}{44.5}} &21.3 \\
& &PGD(0.024) &57.5 &\textbf{\textcolor{red}{49.0}} &\textbf{\textcolor{red}{32.9}} &\underline{\textcolor{blue}{24.3}} 
    &56.2 &\underline{\textcolor{blue}{47.6}} &17.9 
    &56.0 &\textbf{\textcolor{red}{46.8}} &\textbf{\textcolor{red}{24.4}} \\
& &DF-AT &59.4 &47.1 &\underline{\textcolor{blue}{32.0}} &\textbf{\textcolor{red}{24.6}} 
    &\underline{\textcolor{blue}{57.4}} &\textbf{\textcolor{red}{47.8}} &\textbf{\textcolor{red}{23.6}} 
    &\underline{\textcolor{blue}{56.1}} &44.0 &\underline{\textcolor{blue}{23.4}} \\
& &BAT &\underline{\textcolor{blue}{63.7}} &46.0 &28.1 &22.0 
    &\textbf{\textcolor{red}{60.1}} &43.0 &14.7 
    &\textbf{\textcolor{red}{58.7}} &39.3 &17.1 \\
\hline 
\multirow{25}*{\rotatebox{90}{Black-box}} &
\multirow{5}*{\rotatebox{90}{Mnist}} & Budget     &$-$  &0.1&0.3&0.5
             & $-$& $-$&$-$
             &0.1&0.3&0.5\\
\cline{3-13}
& &FGSM(0.3) &\textbf{\textcolor{red}{99.3}} &\underline{\textcolor{blue}{98.3}} &88.0 &14.6 & & &
    &\underline{\textcolor{blue}{98.2}} &90.5 &16.8 \\
& &PGD(0.3) &\underline{\textcolor{blue}{98.7}} &\textbf{\textcolor{red}{98.4}} &\textbf{\textcolor{red}{97.3}} &26.9& & &
    &\textbf{\textcolor{red}{98.4}} &\textbf{\textcolor{red}{97.8}} &20.2 \\
& &DF-AT &97.9 &97.2 &90.0 &\textbf{\textcolor{red}{37.0}} & & &
    &97.3 &92.7 &\textbf{\textcolor{red}{64.3}} \\
& &BAT &98.5 &97.8 &\underline{\textcolor{blue}{92.7}} &\underline{\textcolor{blue}{35.2}} & & &
    &97.7 &\underline{\textcolor{blue}{94.0}} &\underline{\textcolor{blue}{53.7}} \\
\cline{2-13}

& \multirow{5}*{\rotatebox{90}{Fashion}} & Budget    &$-$   &0.05&0.1&0.2 &$-$ &$-$ &$-$
             &0.05&0.1&0.2\\
\cline{3-13}
& &FGSM(0.1) &\textbf{\textcolor{red}{92.5}} &79.1 &68.9 &25.5  & & &
    &67.5 &59.0 &17.5 \\
& &PGD(0.1)  &85.7 &\underline{\textcolor{blue}{84.3}} &\textbf{\textcolor{red}{82.9}} &37.3   & & &
    &\underline{\textcolor{blue}{84.3}} &\textbf{\textcolor{red}{83.4}} &45.7\\
& &DF-AT  &87.5 &82.0 &78.3 &\textbf{\textcolor{red}{64.8}}  & & &
    &82.3 &79.4 &\textbf{\textcolor{red}{72.2}}  \\
& &BAT &\underline{\textcolor{blue}{88.8}} &\textbf{\textcolor{red}{86.4}} &\underline{\textcolor{blue}{81.2}} &\underline{\textcolor{blue}{53.3}}  & & &
    &\textbf{\textcolor{red}{86.3}} &\underline{\textcolor{blue}{82.6}} &\underline{\textcolor{blue}{64.5}}  \\
\cline{2-13}

& \multirow{5}*{\rotatebox{90}{SVHN}} & Budget     &$-$    &0.001&0.005&0.01 &$-$ &$-$ &$-$
             &0.001&0.005&0.01 \\
\cline{3-13}
& &FGSM(0.024)  &\textbf{\textcolor{red}{88.3}} &79.1 &44.8 &21.6 & & &
    &76.3 &52.0 &22.7  \\
& &PGD(0.024)  &82.0 &79.2 &\textbf{\textcolor{red}{64.7}} &\textbf{\textcolor{red}{42.3}}  & & &
    &79.4 &\textbf{\textcolor{red}{66.0}} &\textbf{\textcolor{red}{49.6}}  \\
& &DF-AT  &\underline{\textcolor{blue}{85.5}} &\underline{\textcolor{blue}{79.8}} &\underline{\textcolor{blue}{48.2}} &\underline{\textcolor{blue}{32.3}}  & & &
    &\underline{\textcolor{blue}{80.2}} &\underline{\textcolor{blue}{55.2}} &\underline{\textcolor{blue}{39.9}} \\
& &BAT &\textbf{\textcolor{red}{88.3}} &\textbf{\textcolor{red}{82.8}} &46.5 &22.8  & & &
    &\textbf{\textcolor{red}{82.0}} &36.3 &14.1 \\
\cline{2-13}

& \multirow{5}*{\rotatebox{90}{Cifar10}} & Budget     &$-$    &0.05&0.1&0.2 &$-$ &$-$ &$-$
             &0.05&0.1&0.2 \\
\cline{3-13}
& &FGSM(0.024)  &\underline{\textcolor{blue}{84.9}} &\underline{\textcolor{blue}{80.9}}  &\underline{\textcolor{blue}{70.8}} &\underline{\textcolor{blue}{42.6}}  & & &
    &\textbf{\textcolor{red}{83.8}} &\underline{\textcolor{blue}{81.1}} &\underline{\textcolor{blue}{69.4}}  \\
& &PGD(0.024)  &84.0 &\textbf{\textcolor{red}{81.2}}  &\textbf{\textcolor{red}{72.4}} &\textbf{\textcolor{red}{42.8}}  & & &
    &\underline{\textcolor{blue}{83.1}} &\textbf{\textcolor{red}{81.5}} &\textbf{\textcolor{red}{72.0}}  \\
& &DF-AT  &81.3  &{77.4} &{62.9} &{37.4}  & & &
    &80.2 &77.3 &{62.8} \\
& &BAT &\textbf{\textcolor{red}{85.8}}  &{80.3} &{65.8} &{35.8}  & & &
    &{83.4} &80.3 &65.2 \\
\cline{2-13}

& \multirow{5}*{\rotatebox{90}{Cifar100}} & Budget      &$-$   &0.001&0.005&0.01 &$-$ &$-$ &$-$
             &0.001&0.005&0.01 \\
\cline{3-13}
& &FGSM(0.024)  &58.2 &\underline{\textcolor{blue}{57.6}} &\underline{\textcolor{blue}{53.0}} &\underline{\textcolor{blue}{33.9}}  & & &
    &57.7 &56.7 &53.2 \\
& &PGD(0.024)  &57.5 &{57.3} &\textbf{\textcolor{red}{53.2}} &\textbf{\textcolor{red}{34.2}}  & & &
    &57.2 &{56.6} &\underline{\textcolor{blue}{53.5}} \\
& &DF-AT  &\underline{\textcolor{blue}{59.4}} &57.3 &\textbf{\textcolor{red}{53.2}} &\textbf{\textcolor{red}{34.2}}  & & &
    &\underline{\textcolor{blue}{59.0}} &\underline{\textcolor{blue}{57.4}} &\textbf{\textcolor{red}{53.9}} \\
& &BAT &\textbf{\textcolor{red}{63.7}} &\textbf{\textcolor{red}{62.7}} &50.1 &23.6  & & &
    &\textbf{\textcolor{red}{63.0}} &\textbf{\textcolor{red}{61.0}} &52.8 \\
\hline
\end{tabular}
}
\end{table*}  

\subsection{Overall Results}

\textbf{White-box attack} results for BAT are shown in Figure~\ref{fig:w-bat} and Table~\ref{compare} (facing $\ell_2$ attacks) \& Table~\ref{compare1} (facing $\ell_\infty$ attacks), and are compared with the results obtained by the NT, DF-AT, FGSM-AT and PGD-AT approaches in each subfigure. 
We choose the most widely used prescribed budgets for FGSM-AT/PGD-AT in the literature~\cite{2014arXiv1412.6572G,madry2017towards,song2018improving}, corresponding to the model with a relatively good robustness.
For all of these results, the BAT approach achieved good comprehensive robustness and our results generally among the highest-quality results and are shown as the black solid lines in Figure~\ref{fig:w-bat}. 
Comparison of the robustness under different AT approaches and attack methods shows that the use of the same AEs in both generation and evaluation often shows good performance, such as for evaluating FGSM-AT by the FGSM attack, PGD-AT by PGD and DF-AT by DeepFool.
By contrast, using different kinds of AEs in the generation and evaluation may show relatively poor results. 
As we expected, BAT can achieve a comprehensive robust model, see Table~\ref{compare} for the comparison of \textbf{average adversarial accuracy} of DeepFool as Equation~\eqref{eq:aaa}. 
We set $\Theta = \mathbb{E}||\delta(\mathbf{x})||$ and the half of it, due to application issues are more concerned with the robustness facing AEs do not exceed decision boundaries.
It is clear that the avg-$\A\A(\Theta)$ of BAT are \textbf{in the top two of all datasets and ranking first in most cases}, show that our method can balance the clean accuracy and the comprehensive robustness of the model.
DeepFool-AT is good at defending against large attacks, so the score for large $\Theta$ is slightly higher than BAT in some cases. 

Next, we will provide a specific analysis for each AT approach. 
The FGSM-AT approach is almost entirely unable to defend against attacks other than the FGSM attack. 
The PGD-AT approach is slightly better for the FGSM/PGD attack but is not good enough when confronted with the Noise, DeepFool and CW attacks.
The DF-AT approach can achieve models with good robustness, particularly when facing DeepFool and CW attacks with large attack strengths, but the clean data accuracy is greatly reduced, and the $\A\A$ is lower under small attacks. 
For all of the datasets, the NT approach always attains the highest clean data accuracy, while the BAT approach is second only to NT and achieves a fairly high clean data accuracy compared with the other AT approaches, as can be clearly observed from an examination of the data presented in Table~\ref{compare}. 
Accompanied by the high clean data accuracy, BAT returns high-level $\A\A$ with a small attack strength, as observed from the data shown in red bold or blue underline in Table~\ref{compare}, along with a slight loss on $\A\A$ when facing a large attack strength compared to that of DF-AT. The FGSM, Noise and PGD used in the results are based on the $\ell_\infty$ norm, BAT/DF-AT has zero knowledge about the attack budgets, thus the comparison with restricted AT (FGSM/PGD-AT) is slightly unfair for BAT/DF-AT. It is observed that even though BAT is based on the $\ell_2$ norm, it still shows a fairly high robustness on the $\ell_\infty$ norm in most cases.

\textbf{Black-box attack} results for BAT are also shown in Figure~\ref{fig:w-bat} and Table~\ref{compare} (facing $\ell_2$ attacks) \& Table~\ref{compare1} (facing $\ell_\infty$ attacks). 
For fairness and realism, we use NT models to generate black-box AEs for attacking other approaches. Here, the $\A\A$ is computed under black-box AEs with varying attack strengths. 
In fact, the Noise attack in the white-box attack is also a black-box attack. 
Overall, it is clear that the $\A\A$ of the black-box AEs should be greater than that of the white-box AEs.
The results show that the BAT approach can provide a better defense against black-box attacks in most cases, as is particularly pronounced for the FGSM and PGD attacks.
For the DeepFool and CW attacks, most AT approaches can achieve high-level $\A\A$. Because the NT models are too vulnerable to these two attacks, these AEs from NT models always have norms that are too small and can be easily defended by the AT models.

\subsection{Influence of Cutoff and Scale Strategies}

\begin{figure*}[t]
  \centering
  \includegraphics[width=0.8\textwidth]{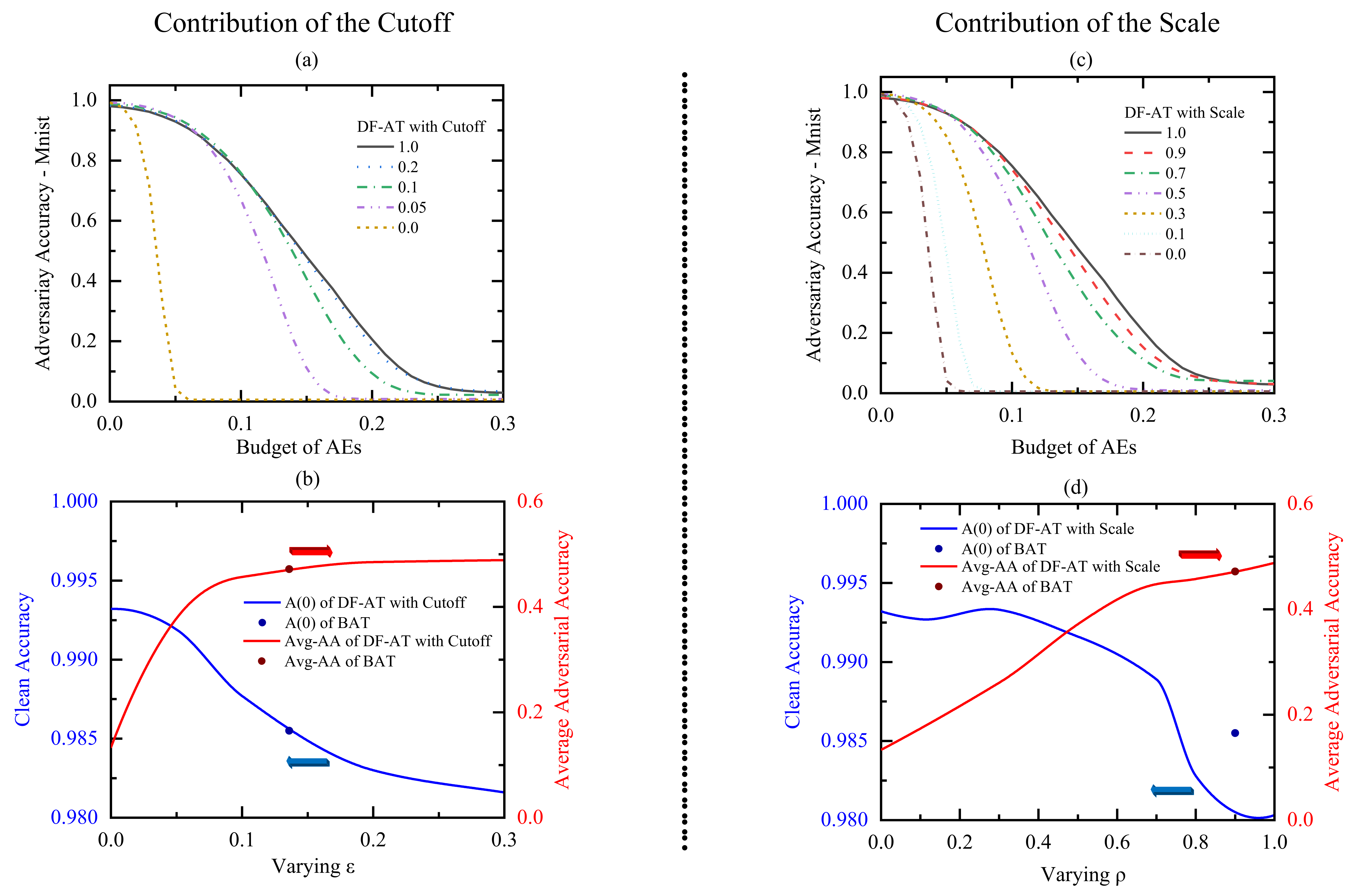}
  \caption{Influence of the cutoff~(varying $\varepsilon$) and scale~(varying $\rho$) strategies. 
  Subfigures (a-d) show the plots of $\A\A$, clean data accuracy and avg-$\A\A(0.3)$, respectively, obtained by sequentially varying $\varepsilon$ and $\rho$.
   The solid dots in subfigures (b,d) show the clean accuracy and avg-$\A\A(0.3)$ of BAT with the budget adaptively estimated using $\varepsilon$ and $\rho$, respectively.}
  \label{fig:df_cos}
\end{figure*}

We test the influence of the cutoff and scale on MNIST. The results for the influence of the cutoff are shown in Figure~\ref{fig:df_cos} (a-b), with $\rho = 1$ such that it is ineffective in terms of the scale part.
 It is clear that with the stricter cutoff parameter $\varepsilon$, the $\A\A$ with the attack strength larger than $\varepsilon$ decreases, but the $\A\A$ with a smaller attack strength do not change significantly. To show the trend of decreasing $\varepsilon$ more clearly, we plot the clean data accuracy $\A(0)$ and avg-$\A\A(0.3)$ 
in Figure~\ref{fig:df_cos} (b). We choose $\Theta=0.3$ to elucidate the influence of the $\varepsilon$ variation on the comprehensive robustness, since the $\A\A$ is almost $0$ if the attack strength $>0.3$.
The results clearly support our intuition that as $\varepsilon$ is adjusted from $1$ to $0$, avg-$\A\A(0.3)$ decreases and $\A(0)$ increases. 
Based on this, we also plot the $\A(0)$ and avg-$\A\A(0.3)$ of BAT (solid dots) in Figure~\ref{fig:df_cos} (b). 
Surprisingly, the results show that our BAT approach can adaptively estimate a good budget and obtain the model with both high clean data accuracy and comprehensive robustness. If the budget is adaptively estimated to be slightly larger, the avg-$\A\A(0.3)$ does not grow strongly but $\A(0)$ sharply declines; however, if the budget is adaptively estimated to be slightly smaller, the avg-$\A\A(0.3)$ will decrease sharply.

The results for the influence of the scale are shown in Figure~\ref{fig:df_cos} (c-d), where we set $\varepsilon = 1$ such that it is ineffective in terms of the cutoff part. 
The results show a trend similar to the trend observed for the cutoff, but as $\rho$ decreases, the $\A\A$ of the whole area declines proportionately, which is consistent with the scale principle. As shown in Figure~\ref{fig:df_cos} (d), as $\rho$ decreases, the avg-$\A\A(0.3)$ shows a relatively uniform downward trend, and $\A(0)$ is more sensitive for a large $\rho$. This feature encourages us to choose $\rho=0.9$ in this work, and the results for $\A(0)$ and avg-$\A\A(0.3)$ of BAT (solid dots) shown in Figure~\ref{fig:df_cos} (d) also support this choice.

In fact, from another point of view, we may also consider BAT in \eqref{eq:opt_bat} as a comprehensive approach of NT (the training without using AEs) and DF-AT in \eqref{eq:opt_df_at}. We indicate that the BAT approach with cutoff $\varepsilon= 1$ and scale $\rho = 1$ reduces to the DF-AT approach because there is no change to the AEs from DeepFool. Moreover, the BAT approach with $\varepsilon = 0$ or $\rho = 0$ is identical to the NT approach, where the zero parameter of the cutoff or scale compresses the AEs back to the clean data. Therefore, the cutoff and scale strategy can connect and carry out the transition from DF-AT to NT with adjusting parameters $\varepsilon$ and $\rho$. The numerical results in Figure~\ref{fig:df_cos} show that the DF-AT approach can obtain a model with high robustness when encountering large attack strength, while the NT approach can clearly produce a model with high clean data accuracy. The proposed cutoff and scale strategy in BAT attempts to provide an approach combining the advantages of both DF-AT and NT, dynamically adjust a nonuniform budget in blind, and the model search has relatively high clean data accuracy and comprehensive robustness at the same time.

\section{Conclusion and Future Work}

Both restricted and unrestricted AT approaches cannot obtain a comprehensively robust model. To alleviate this problem, this paper proposes the \textbf{blind adversarial training (BAT)} approach that ameliorates the problem by using the \textbf{cutoff-scale} strategy in order to adaptively estimate a nonuniform budget in the generation of each AE during the training; this approach tends to obtain an comprehensively robust model,  where \textbf{adversarial accuracy} is used as the  measure for the evaluation of comprehensive robustness. 
By using BAT to train the classification models on several benchmarks, we obtain comprehensive robust models with both high accuracy and robustness for various white/black-box attacks with multiple attack methods and varying attack strengths. The individual contribution of the cutoff and scale are well-addressed for the MNIST dataset. In summary, the proposed BAT approach shows competitive performance. 

The present research is still in the early stage. There are several aspects that deserve deeper investigation:
\begin{itemize}
	\item Theoretical analysis of BAT behavior for the more general classification problems, like two circles classification problem defined in Figure~\ref{fig:tcc}; 
	\item Design other kinds of adaptive budgets to improve the cutoff strategy, like using DDN~\cite{DBLP:journals/corr/abs-1811-09600} instead of DeepFool to generate AEs in BAT;
	\item Investigate the performance of BAT on more complex models or datasets, such as ImageNet;
	\item Combine BAT and model pruning, i.e. adaptively generating AEs into the model pruning process, to develop an approach that can better balance the accuracy, robustness and efficiency. 
\end{itemize}

\section*{Acknowledgements}

This work was supported in part by the Innovation Foundation of Qian Xuesen Laboratory of Space Technology, and in part by Beijing Nova Program of Science and Technology under Grant Z191100001119129. 

\bibliographystyle{ieee_fullname}
\bibliography{paper_ref}

\clearpage
\setcounter{prop}{0}
\section*{{Appendix}}
\appendix 

In order to better support the views and results of the paper, we elaborate this appendix.
In Section~\ref{app:notation}, the notations and definitions used in main text are listed.
In Section~\ref{app:tcc}, the two circles classification problem is analyzed in detail.
In Section~\ref{app:result}, we list and compare all numerical results of datasets and attack methods mentioned in main text.
In Section~\ref{app:prove}, Given the mathematical proof of the three propositions in Section~2 of main text.

\section{Notations and Definitions}\label{app:notation}

Table~\ref{tab:Notations}~\&~\ref{tab:Definitions} list the notations and definitions used in main text of this paper.

\begin{table}[ht]
\caption{Notations}
\label{tab:Notations}
\centering
\begin{tabular}{c c}
\toprule
Notations & Full name \\
\midrule
AEs & Adversarial Examples  \\
AT  & Adversarial Training \\
NT  & Normal Training \\
BAT & Blind Adversarial Training  \\
TCC & Two Circles Classification  \\
CoS & Cutoff-Scale  \\
AA  & Adversarial Accuracy  \\
avg-AA & Average Adversarial Accuracy  \\
FGSM & Fast Gradient Sign Method  \\
PGD & Projected Gradient Descent  \\
CW  &Carlini and Wagner Attacks method  \\
DF  & DeepFool  \\
$\mathcal{L}(\theta,\mathbf{x},\mathbf{y})$ & Loss function \\
$\delta(\mathbf{x})$ & Adversarial perturbation \\
$\eta$    &  Prescribed uniform budget of AT \\
$\varepsilon$    &  Cutoff parameter of BAT \\
$\rho$    & Scale parameter of BAT \\
\bottomrule
\end{tabular}
\end{table}

\begin{table}[ht]
\caption{Definitions}
\label{tab:Definitions}
\centering
\begin{tabular}{c c}
\toprule
Names & Definitions \\
\midrule
\multirow{3}{*}{Budget} &   A parameter to constrain \\
                        &   the magnitude of AE \\
                        &   perturbation during AT.  \\
\hline
\multirow{3}{*}{Norm}   & The norm of the AE \\
                        & perturbation, include \\
                        & $\ell_0$, $\ell_2$ and $\ell_\infty$ etc. \\
\hline
\multirow{4}{*}{Weak attack}  & Attack the model by AEs \\
                              & with small norm, expression \\
                              & as the AEs is far from \\
                              & reaching the decision boundary. \\
\bottomrule
\end{tabular}
\end{table}

\section{Supplementary Numerical Results for the Two Circles Classification Problem}
\label{app:tcc}

\begin{figure*}[ht]
  \centering
  \includegraphics[width=1\textwidth]{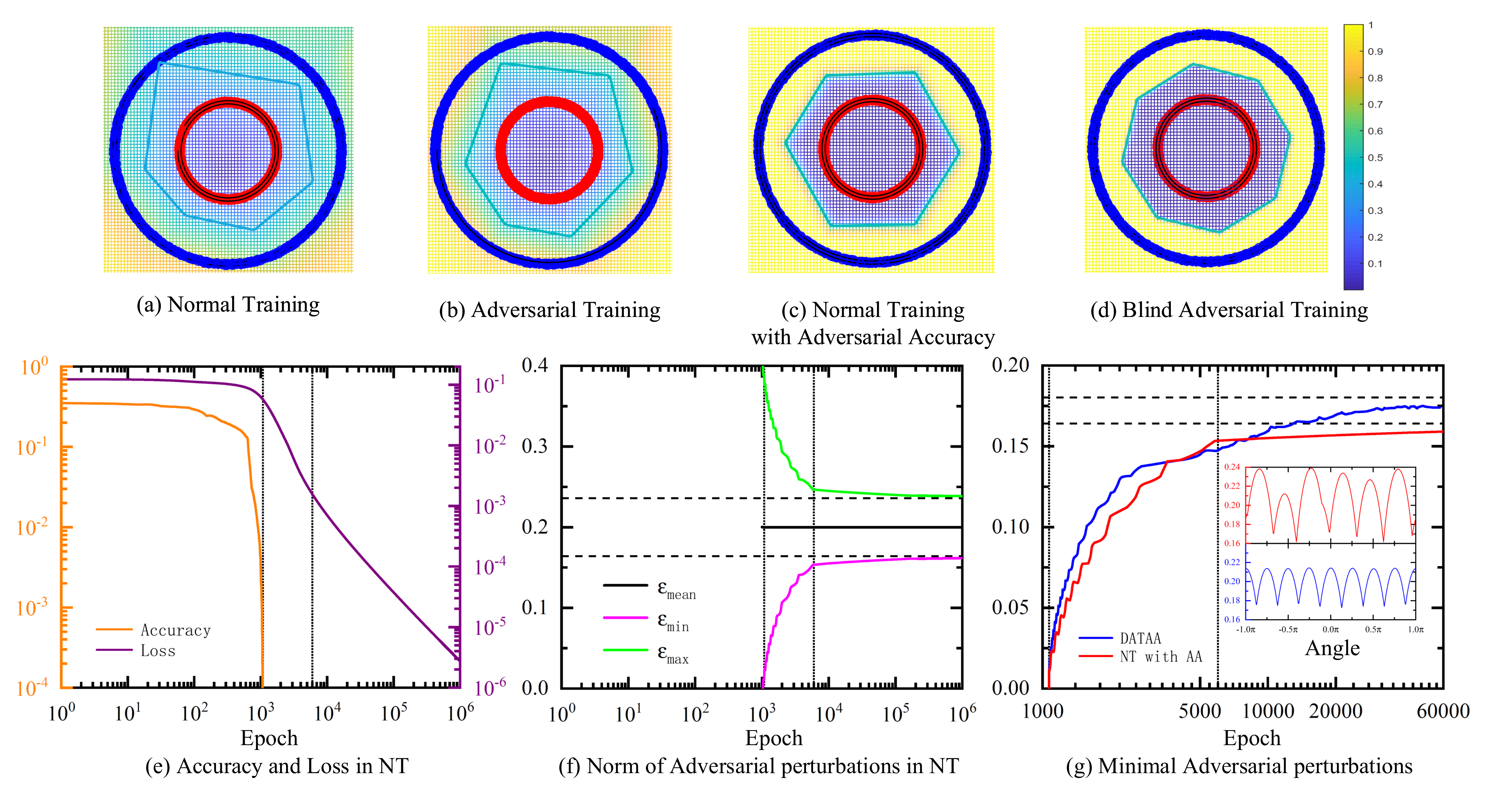}
  \caption{Comparison of different training approaches employing the two circles classification problem with the $1$-hidden-layer $6$-dim perceptron, learning rate $\alpha=0.2$, the radius of a circles $r_1=0.3, r_2=0.7$.    
    (a-d) Compare the result of normal training (NT), adversarial training with a budget of $0.1$, NT with $\A\A$ as convergence cohesion and blind adversarial training. The solid blue/red lines correspond to the datasets of two labels; the solid gray lines represent the decision boundary of the classifiers.
    (e) Accuracy and Loss in NT converge with epochs. 
    (f) Maximum, minimum and average of adversarial perturbation in NT converge with epochs. 
    (g) Compare the robustness converge with epochs and (inner) the exact perturbation at each angle (from the polar coordinate of the circles) by using NT with $\A\A$ (c) and BAT (d).}
  \label{fig:tcc_app}
\end{figure*}

\paragraph{Model overview:} As discussed in Section~1 (main text), we using BAT into the two circles classification (TCC) problem in two dimension. The target of TCC problem is to train a classifier to distinguish two concentric circles as shown in figure~\ref{fig:tcc_app}~(a). We randomly sample $5000$ in train-data and $1000$ in test-data which is large enough for the TCC problem, and use the simplest $1\text{-hidden-layer } 6 \text{-dim}$ perceptron as classifier. With the input coordinate $\mathbf{x}_1,\mathbf{x}_2 \in [-1,1]$ and output label $\mathbf{y} \in [0,1]$, the boundary at $\mathbf{y} = 0.5$ is a polygon on the two dimensional plane, 
which ideally approaching to regular polygon close to the mid-circle of two labels with radius as $r = \frac{r_1 + r_2}{2}$. The loss of TCC problem is
\begin{equation}
L(\theta,\mathbf{x},\mathbf{y}) =  \E_\mathbf{x}\,  (C(\mathbf{x})-\mathbf{y})^2
\label{eq:loss_tcc}
\end{equation}
\paragraph{Robustness:} TCC problem is analytical with $100\%$ accuracy, adversarial perturbations $\varepsilon(\mathbf{x})$ can be exact calculated by the nearest distance between input-data $\mathbf{x}$ and the decision boundary of the model,
$\varepsilon(\mathbf{x}) = \min || C^{-1}(\mathbf{y}=0.5) - \mathbf{x} ||$, where $C^{-1}(\mathbf{y}=0.5)$ is the decision boundary of the classifier $C$, i.e. the polygons between the two circles in Figure~\ref{fig:tcc_app}~(a-d). 
Define $\eta$ as the minimum of perturbations, $\eta = \min \varepsilon(\mathbf{x})$, during CoS strategy.
Since the $\A\A(\eta)$ is exactly $100\%$, higher $\eta$ means more robustness.
The $\A\A$ in TCC problem is replaced by $\eta$.

\paragraph{Training strategy:} We compare normal training (NT), Adversarial Training (AT), NT with $\A\A$ and Blind Adversarial Training (BAT) in Figure~\ref{fig:tcc_app}~(a-d). 
Figure~\ref{fig:tcc_app}~(a) lists an example of a model by NT with accuracy $100\%$ but with poor robustness. 
AT could improve a little but still far from its best, NT with $\A\A$ shows better result with closer to regular hexagon than AT, while BAT could get the best result and further push the model close to the regular octagon. 
These results clearly demonstrate the individual contribution of CoS and $\A\A$. 

\paragraph{Contribution of $\A\A$:} To clarify the comparison, we investigate the details of the training process. 
If we train the model until accuracy $100\%$, see the left dotted black line in Figure~\ref{fig:tcc_app}~(e), we may get the corresponding model with poor robustness, see Figure~\ref{fig:tcc_app}~(a). 
Then if we use $\A\A$ to monitor the training, when $\A\A$ converge, see the right dotted black line in Figure~\ref{fig:tcc_app}~(e), we may get a better model, see Figure~\ref{fig:tcc_app}~(c). While at the same time, the adversarial perturbations correspondingly change as that in Table~\ref{Comparison}, from which we found $\A\A$ will lead the model to the one with about maximum minimal attack cost. That may guarantee the individual contribution of $\A\A$. 

\paragraph{Contribution of CoS:} However, even though $\A\A$ converges, the minimal adversarial perturbation have not reached its maximum if we keep training with huge number of epochs, see the curves of minimal adversarial perturbation in Figure~\ref{fig:tcc_app}~(f). But the $10^6$ epochs is unacceptable for actual training, especially it may lead over-fitting for more practical problems.
The reason of $10^6$ epochs requiring may due to that the data is far from the boundary of the network after $6000$ epochs and then its contribution of enhancing network's robustness is weakness. While AT may relieve this but its improvement is still limited since the threshold is fixed, see Figure~\ref{fig:tcc_app}~(b). The proposed CoS strategy will dynamically adjust the threshold, pushing the model quickly flows to its best position, see Figure~\ref{fig:tcc_app}~(d,g). From Figure~\ref{fig:tcc_app}~(g), BAT converges faster and still has outstanding improvement after $6000$ epochs, compared with NT with $\A\A$. We also plot the exact perturbation at each angle by NT with $\A\A$ and BAT, the result of BAT more robustness than NT with $\A\A$, which can be understand as the force exerted by CoS-AEs on model is more powerful than clean data, as CoS-AEs can keep the manifold of clean data, it can push the model to a better position. These results clearly demonstrates the individual contribution of CoS. 

\paragraph{Result comparison:} Table~\ref{Comparison} compares the robustness (the minimal norm of adversarial perturbation) of several training and evaluation methods for TCC problem that is corresponding to Figure~\ref{fig:tcc_app}~(a-d), where the Hexagon and Octagon represent the theoretical results when the boundaries of networks are exact Hexagon and Octagon in the middle of two circles. NT only get $100\%$ accuracy but can not guarantee the robustness. AT is better than NT but only can defense the adversarial attacks below to trained threshold. From Table~\ref{Comparison} we can find that $\A\A$ will improve the robustness of NT and AT, even close to the result of the Hexagon, but converge slower and slower due to the gradient descent quickly. While the proposed CoS (BAT) will leap the model towards Octagon, get most robustness result then others.

\begin{table}[ht]
\caption{Compare the robustness (the minimal norm of adversarial perturbation) of several training and evaluation methods.}
\label{Comparison}
\centering
\begin{tabular}{ccc}
\toprule
Training  & Evaluation  & Robustness \\
\midrule
NT         & Acc.        & 0    \\
AT         & Acc.        & 0.13 \\
NT         & $\A\A$          & 0.162\\
AT         & $\A\A$          & 0.162\\
Hexagon    & Upper bond                       & 0.164\\
\textbf{BAT}      & \textbf{$\A\A$}          & \textbf{0.175}  \\
Octagon    & Upper bond                       & 0.180\\
\bottomrule
\end{tabular}
\end{table}

\section{Supplementary Numerical Results for Benchmark Datasets}\label{app:result}

\subsection{Restricted Adversarial Training}

\begin{figure*}[t]
  \centering
  \includegraphics[width=0.95\textwidth]{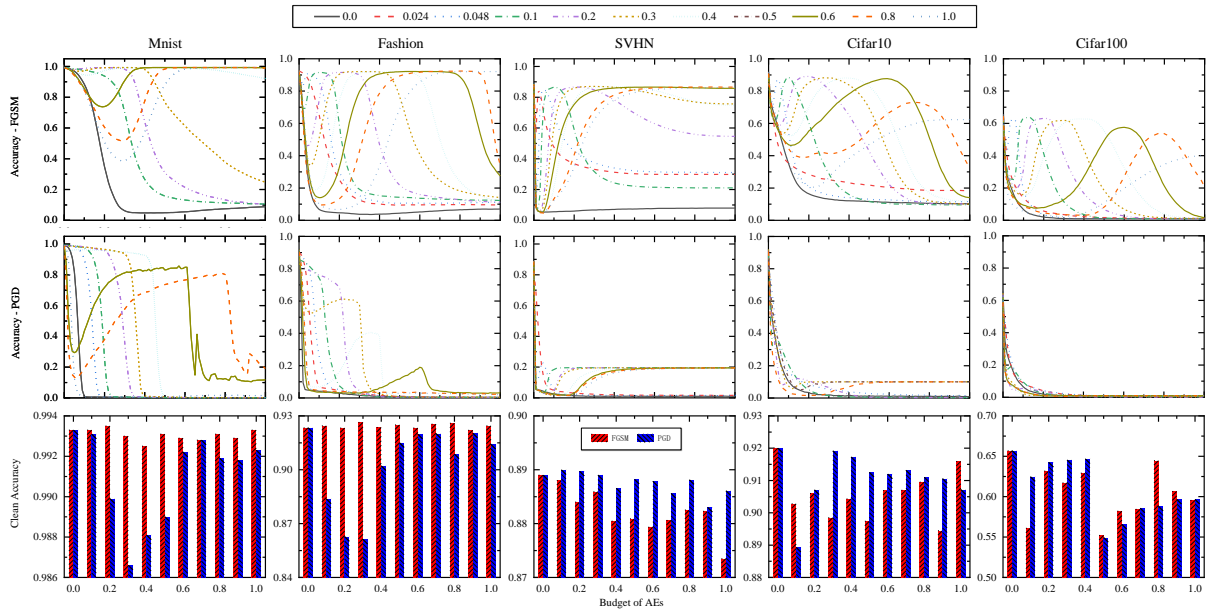}
  \caption{Comparison of the accuracy against attacks with varying budgets (x-axis of the third row) of the AT models (FGSM-AT/PGD-AT) with varying budgets (the lines in the sub-figures of the first two rows) on MNIST, Fashion-MNIST, SVHN, CIFAR10 and CIFAR100. We also show the accuracy on clean data in the sub-figures of the third row. }
  \label{fig:fgsm_pgd_app}
\end{figure*}

As shown in Figure~2 (main text), we train FGSM-AT/PGD-AT models with varying budgets (the lines), and then attack these AT models using AEs generated with varying budgets (the x-axis). Figure~\ref{fig:fgsm_pgd_app} shows more numerical results about FGSM-AT/PGD-AT. 
Analogous to the TCC problem, if we train with a small budget, the AT model is always easily attacked with a large budget; note the drop in the accuracy when the models are trained with a budget larger than $0.5$ for FGSM-AT (the first row of Figure~\ref{fig:fgsm_pgd_app}) and the ones with a budget larger than $0.6$ for PGD-AT (the second row of Figure~\ref{fig:fgsm_pgd_app}). 
In addition, if we train with a large budget, the AT model is even weak when confronting the attack with a small budget. 
Furthermore, the improvement on robustness also has a remarkable influence on the clean data accuracy, as shown in the third row of Figure~\ref{fig:fgsm_pgd_app}.
Based on the existing numerical experiments and references, we can summarize that the model obtained with the prescribed budget is only robust when encountering the attack with the same strength.
The model is obviously weak when facing an attack that is stronger than the prescribed budget used in AE generation and is overly defensive when facing a small attack or on clean data.

\subsection{Time Cost}
All experiments are implemented on single NVIDIA Tesla V100 GPU. For DF-AT and BAT, It cost about a hour's magnitude for training MNIST, Fashion-MNIST and SVHN datasets, and a few hours for training CIFAR10 and CIFAR100. 

As shown in the Algorithm in main text, the main add-ons of BAT is CoS strategy, which is very simple from the point of view of computation cost. Thus the extra expense of BAT than DF-AT is negligible. 

\section{Prove details in Section~2} \label{app:prove}

This section explains in detail the derivation process of the Propositions in Section~2 (main text).

\subsection{Model and basic assumptions}

Most neural networks are highly non-linear and complex, to theoretical analysis AT, we use a simplified $1$-layer perceptron defined as, $\mathbf{y}=\sigma(\mathbf{W} \mathbf{x} + b)$, which is used in TCC problem in Section~$1$ (main text). It is usually use sigmoid activation function $\sigma(x) = \frac{1}{1+e^{-x}}$. In order to study the impact of AEs on neural networks more clearly, we consider a pair of neighbor data $\mathbf{x}^1, \mathbf{x}^2$ within different labels $\mathbf{y}^1, \mathbf{y}^2$, and the corresponding AEs represent as $\mathbf{x}^1_{\rm adv}, \mathbf{x}^2_{\rm adv}$. In this way, we can clearly analysis the impact of the two data, and deduce the general nature of AEs during AT. 

We hope to get a constructive conclusion from the strict extrapolation of the simple problem, and then generalize and expand it to a general model to reveal the effectiveness and significance of the proposed BAT approach.

\begin{figure}[t]
  \centering
  \includegraphics[width=\linewidth]{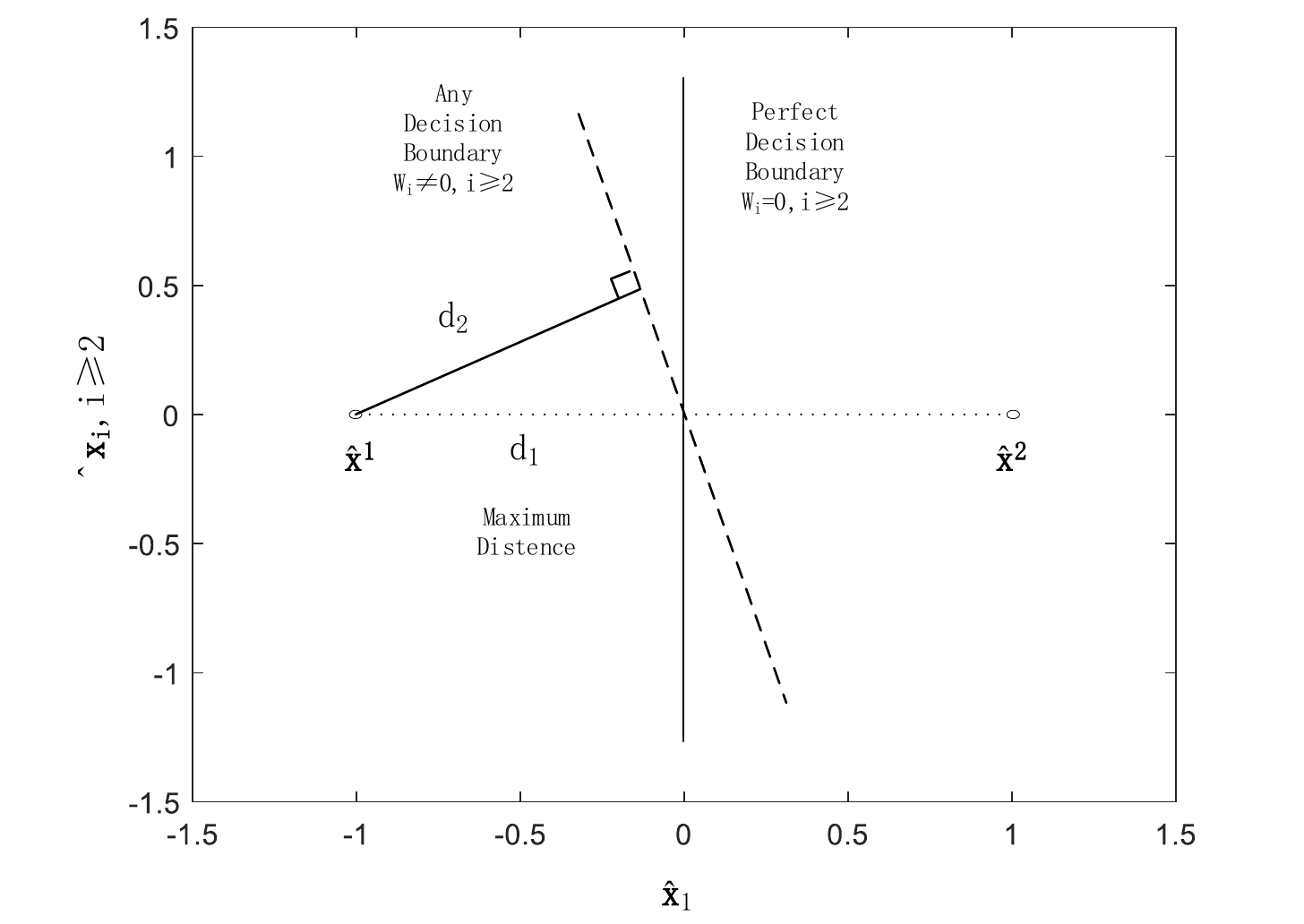}
  \caption{Sketch map of the model, only perfect decision boundary with $W_i = 0, i\geq 2$ has the maximum distance between data and decision boundary.
  }
  \label{fig:model}
\end{figure}

\begin{lemma}
$\forall$ $\mathbf{x}^1, \mathbf{x}^2 \in \R^d$, $\exists$ a linear transformation matrix consist of translation, rotate and scale, $M = T R S$. After transformation $\hat{\mathbf{x}} = M \mathbf{x}$, must have $\hat{\mathbf{x}}^1 = (-1, \hat{0})^\mathrm{T}, \hat{\mathbf{x}}^2 = (1, \hat{0})^\mathrm{T}$, where $\hat{0} \in \R^{d-1}$ is $d-1$ dimensional zero vector.
\end{lemma}
Thus, without loss of generality, we only consider data $\hat{\mathbf{x}}^1, \hat{\mathbf{x}}^2$ with $\mathbf{y}^1=0, \mathbf{y}^2=1$. Now, we can assume the perceptron is monotonic in concerned region $\hat{\mathbf{x}} \in [\hat{\mathbf{x}}^1, \hat{\mathbf{x}}^2]$ corresponding $\mathbf{y} \in [0,1]$. 

\subsection{Normal Training}

Let's take the square loss as an example, the loss-function during NT can be written as,
\be
\mathcal{L}_{\rm NT} = \min_{\mathbf{W},b} \left\{ \right. \sum_{i=1,2} 
[\sigma(\mathbf{W} \mathbf{x}^i + b) - \mathbf{y}^i]^2 \nonumber\\
+ \lambda || \mathbf{W} ||_2 + \lambda || b ||_2 \left. \right\},
\ee
where $\lambda$ is Lagrange multiplier for regularization and be very small, due to the data has no constraints on the subsequent components of parameter $\mathbf{W}$($W_i, i\geq 2$). It is easy to bring in data and get the minimum value by,
\be
\frac{\partial \mathcal{L}_{\rm NT}}{\partial b} &=& 
2[\sigma(W_1  (-1) + b) - 0]
\sigma^\prime(W_1  (-1) + b) \nonumber\\
&+& 2[\sigma(W_1  (1) + b) - 1]
\sigma^\prime(W_1  (1) + b)  \nonumber\\
&+& \lambda{\rm sign}(b) = 0\\
\frac{\partial \mathcal{L}_{\rm NT}}{\partial W_1} &=& 
2[\sigma(W_1  (-1) + b) - 0]
\sigma^\prime(W_1  (-1) + b) (-1) \nonumber\\
&+& 2[\sigma(W_1  (1) + b) - 1]
\sigma^\prime(W_1  (1) + b) (1)   \nonumber\\
&+& \lambda{\rm sign}(W_1) = 0\\
\frac{\partial \mathcal{L}_{\rm NT}}{\partial W_i} &=& \lambda{\rm sign}(W_i), i\geq 2. \label{eq:nt_w2}
\ee

Obviously, with the derivative form of activation function $\sigma^\prime(x) = \sigma(x)(1-\sigma(x))$, and $\sigma(-x) = (1-\sigma(x))$. Simultaneous these equations, always satisfied $\sigma(x)>0$ and $\sigma^\prime(x)>0$, $W_1$ and $b$ can not be non-zero at the same time. If $W_1=0$, $b \neq 0$, the equations has no solution, thus the minimum point is satisfied only when
\be
b &=& 0 \\
W_1 &{\rm satisfy}& \lambda(1+e^{W_1})^2(1+e^{-W_1})=4 \label{lambda}\\
W_i &=& 0, i\geq 2. \label{eq:nt_w2_value}
\ee
From data $\hat{\mathbf{x}}^1, \hat{\mathbf{x}}^2$, it is intuitive to know the idea decision boundary is the line of 
\be
\hat{\mathbf{x}}_{1} = \frac{\hat{\mathbf{x}}^1_1 + \hat{\mathbf{x}}^2_1}{2} = 0.
\ee
Mathematically the decision boundary defined as 
\be
\{\hat{\mathbf{x}} | \sigma(\mathbf{W} \hat{\mathbf{x}} + b)=\frac{1}{2} \},
\ee
which can be reduced to 
\be
\mathbf{W} \hat{\mathbf{x}} + b = 0.
\ee
\begin{figure}[ht]
  \centering
  \includegraphics[width=\linewidth]{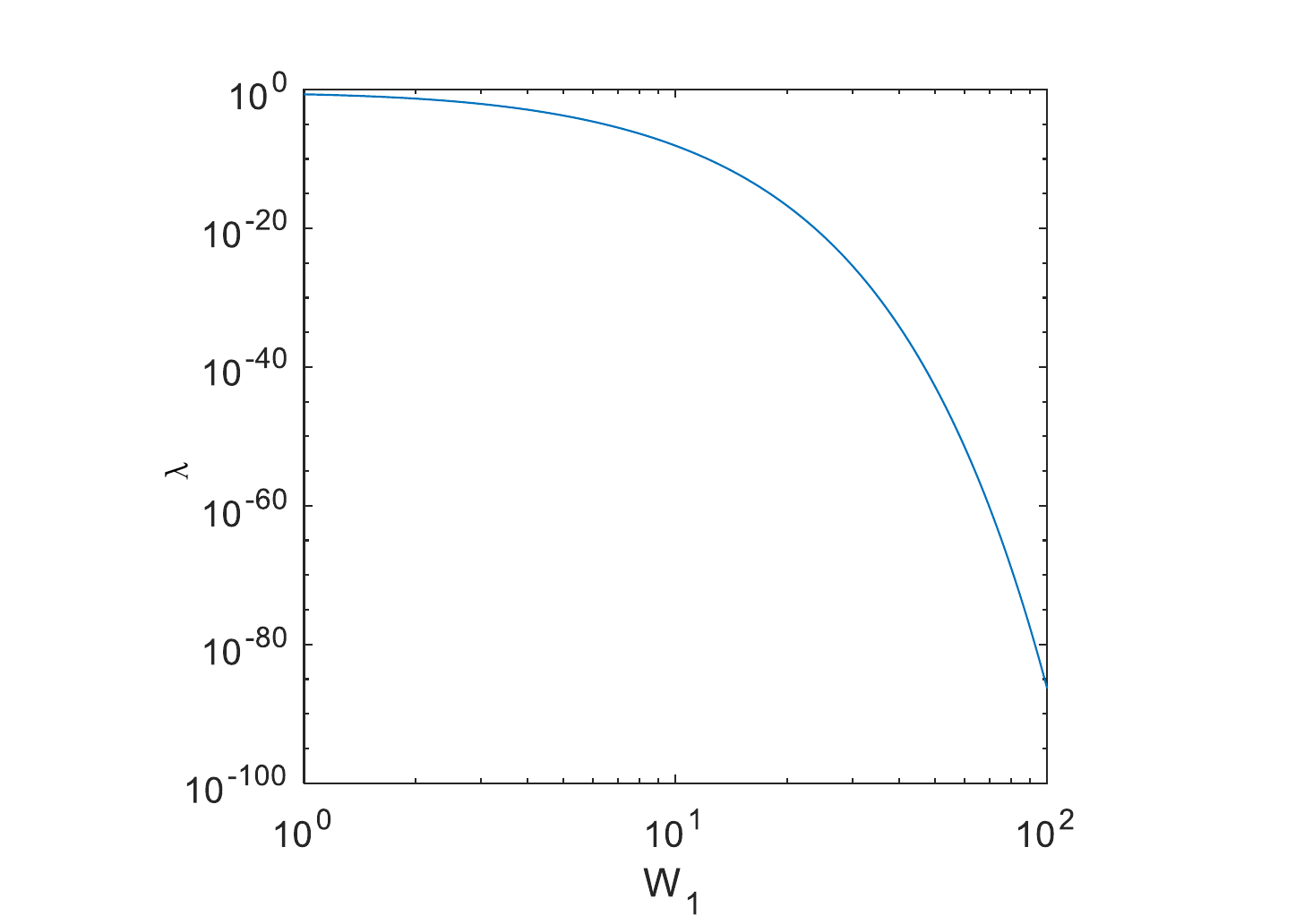}
  \caption{The value of $\lambda$ and $W_1$ that satisfies the Equation~\eqref{lambda}.
  }
  \label{fig:lambda}
\end{figure}

For the result of NT, bring in the minimum point and we can know that the convergence limit of NT is corresponding to the idea classifier with the perfect decision boundary. 
As shown in Figure~\ref{fig:model}, only the line $\hat{\mathbf{x}}_{1} = 0$ can make the network have the maximum adversarial perturbation. 
Different $\lambda$ produces different $W_1$ as in Figure~\ref{fig:lambda}, and the larger $W_1$ lead the output of network more close to $0$ and $1$, which means the larger confidence. 

\subsection{Restricted Adversarial Training}

\begin{lemma}
Gradient based constrained AEs generating methods generate the AEs along $\hat{\mathbf{x}}_1$ axis. Written as,
\be
\hat{\mathbf{x}}^1_{\rm adv} &=& (1 - \eta) \hat{\mathbf{x}}^1 = (-1 + \eta, \hat{0})^\mathrm{T}, \\
\hat{\mathbf{x}}^2_{\rm adv} &=& (1 - \eta) \hat{\mathbf{x}}^2 = (1 - \eta, \hat{0})^\mathrm{T},
\ee
where $\eta \in [0,2]$ is the adversarial perturbation, constraint the AEs should between the two data $\hat{\mathbf{x}}^1$ and $\hat{\mathbf{x}}^2$.
\end{lemma}
\begin{proof}
While usual AT approaches take NT with several epochs as starting point, we assume that the $W_i$( $i\geq 2$) has converged quickly. Using FGSM AEs as an example,
\be
\hat{\mathbf{x}}_{\rm adv} - \hat{\mathbf{x}} = \eta \text{sign} (\nabla L(\hat{\mathbf{x}},\hat{\mathbf{y}})),
\ee
from Equation~\eqref{eq:nt_w2}~\&~\eqref{eq:nt_w2_value},  the gradient of loss only have $\hat{\mathbf{x}}_1$ component, $\frac{\partial \mathcal{L}}{\partial \hat{\mathbf{x}}_i} = 0, i \geq 2$. 
\end{proof}

\begin{prop} For classifying two points $\mathbf{x}^1, \mathbf{x}^2$ with a $1$-layer perceptron model, while both restricted AT and NT can obtain the model with the best robustness, the performance of restricted AT is much better than that of NT because the prescribed budget accelerates the learning process.
\end{prop}

\begin{proof}

The loss function during AT now be,
\be
\mathcal{L}_{\rm AT} = \min_{\mathbf{W},b} \left\{ \right. \sum_{i=1,2} 
\frac{1}{2}[\sigma(\mathbf{W} \hat{\mathbf{x}}^i + b) - \mathbf{y}^i]^2 \nonumber\\
+ \frac{1}{2}[\sigma(\mathbf{W} \hat{\mathbf{x}}^i_{\rm adv} + b) - \mathbf{y}^i]^2 \nonumber\\
+ \lambda || \mathbf{W} ||_2 + \lambda || b ||_2 \left. \right\}.
\ee
We have the first derivative of loss function as, 
\be
\frac{\partial \mathcal{L}_{\rm AT}}{\partial b} &=& 
[\sigma(W_1  (-1) + b) - 0]
\sigma^\prime(W_1 (-1) + b) \nonumber\\
&+& [\sigma(W_1  (\eta-1) + b) - 0]
\sigma^\prime(W_1 (\eta-1) + b) \nonumber\\
&+& [\sigma(W_1  (1) + b) - 1]
\sigma^\prime(W_1  (1) + b)  \nonumber\\
&+& [\sigma(W_1  (1-\eta) + b) - 1]
\sigma^\prime(W_1  (1-\eta) + b)  \nonumber\\
&+& \lambda{\rm sign}(b) = 0 \\
\frac{\partial \mathcal{L}_{\rm AT}}{\partial W_1} &=& 
[\sigma(W_1  (-1) + b) - 0]
\sigma^\prime(W_1  (-1) + b) (-1) \nonumber\\
&+& [\sigma(W_1  (\eta-1) + b) - 0] \nonumber\\
&& \sigma^\prime(W_1  (\eta-1) + b) (\eta-1) \nonumber\\
&+& [\sigma(W_1  (1) + b) - 1]
\sigma^\prime(W_1  (1) + b) (1)   \nonumber\\
&+& [\sigma(W_1  (1-\eta) + b) - 1] \nonumber\\
&& \sigma^\prime(W_1  (1-\eta) + b) (1-\eta)   \nonumber\\
&+& \lambda{\rm sign}(W_1) = 0\\
\frac{\partial \mathcal{L}_{\rm AT}}{\partial W_i} &=& \lambda{\rm sign}(W_i), i\geq 2.
\ee
Simultaneous these equations, obviously $W_i=0$($i\geq 2$), and $b=0$ is the solution of the equations. 
\textbf{Thus,  restricted AT can lead the model to the one with the best robustness (best decision boundary).}

Due to $\eta \in [0,2]$, it is easy to know that there is similar perfect classifier with NT, the only different is that $W_1$ has a slight movement. $W_1$ satisfy,
\be
\frac{\partial \mathcal{L}_{\rm AT}}{\partial W_1} &=&  -2(1-\sigma(W_1))^2\sigma(W_1) \nonumber\\
&-& 2(1-\eta)(1-\sigma((1-\eta)W_1))^2\sigma((1-\eta)W_1) \nonumber\\
&+& \lambda{\rm sign}(W_1) =0 \label{eq:eta}
\ee

\begin{figure}[t]
  \centering
  \includegraphics[width=\linewidth]{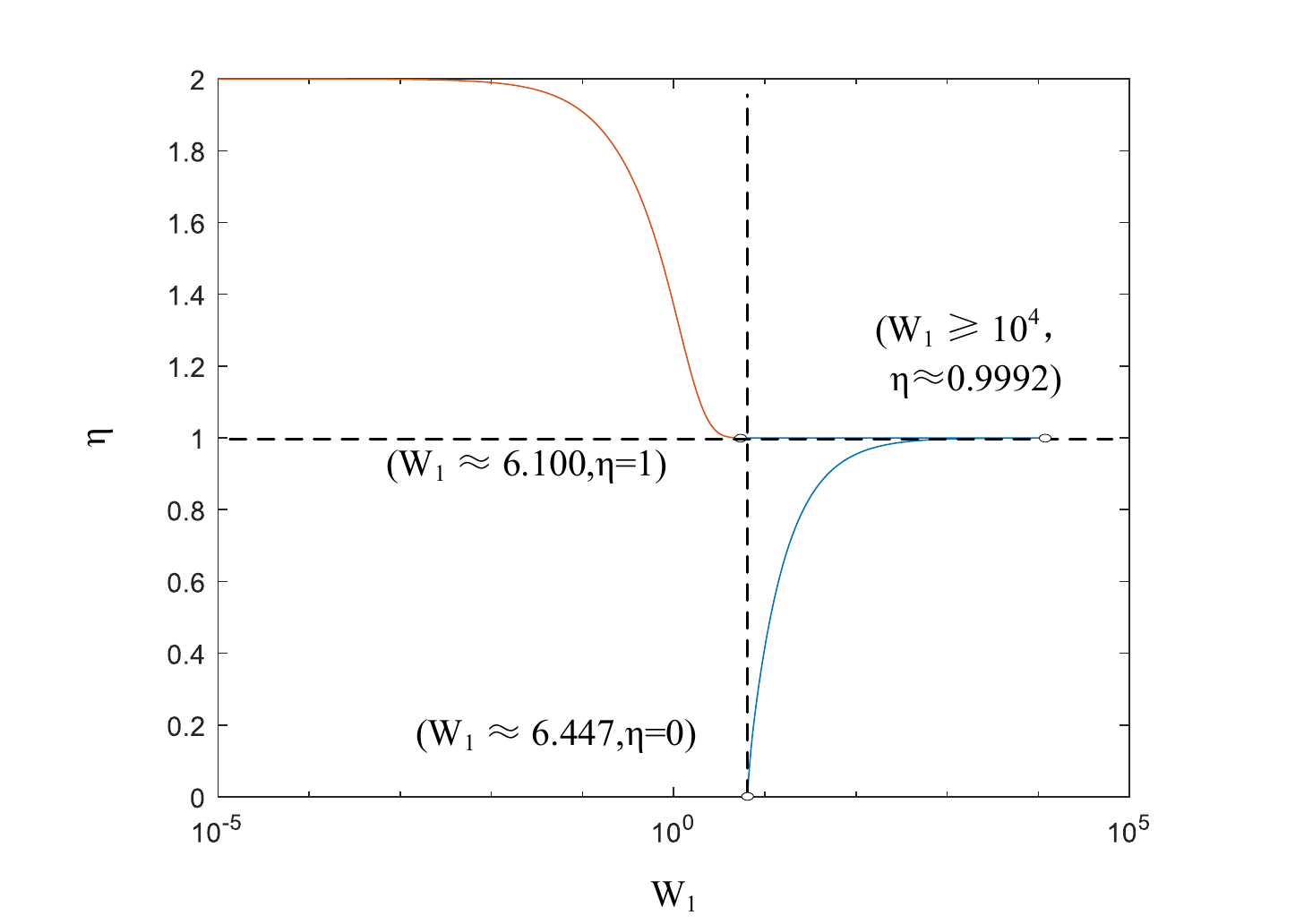}
  \caption{The value of $\eta$ and $W_1$ that satisfies Equation~\eqref{eq:eta} with $\lambda=10^{-5}$. }
  \label{fig:eta}
\end{figure}

It can be found from the solution of the equation, the convergence $W_1$ rapid growth with $\eta$ increases from $0$; until $\eta\approx 0.9992$ (approach but slightly less than $1$), $W_1$ reach its maximum; after that $W_1$ rapid fall down, and when $\eta = 1$, $W_1 \approx 6.100$ which less than $W_1 \approx 6.447$ when $\eta=0$, as shown in Figure~\ref{fig:eta}. \textbf{Thus, there exists a prescribed budget (lead the AEs tending towards the decision boundary) that makes the output model with the largest $W_1$ and best confidence (logit output as close as possible to the real label).} 

Bring in $\sigma^{\prime \prime}(x) = \sigma^{\prime}(x)\sigma(-x)-\sigma(x)\sigma^{\prime}(x) = \sigma^{\prime}(x)(1-2\sigma(x))$. At the minimum point of loss when first derivative be zero, the second derivative of the loss function as follows,
\be
\frac{\partial^2 \mathcal{L}_{\rm AT}}{(\partial b)^2} &=& 
\sigma^\prime(W_1 - b)[1-\sigma(W_1 - b)] \nonumber\\
&& [3\sigma(W_1 - b)-1] \nonumber\\
&+& \sigma^\prime(W_1 + b)[1-\sigma(W_1 + b)] \nonumber\\
&& [3\sigma(W_1 + b)-1] \nonumber\\
&+& \sigma^\prime((1-\eta)W_1 - b)[1-\sigma((1-\eta)W_1 - b)] \nonumber\\
&& [3\sigma((1-\eta)W_1 - b)-1] \nonumber\\
&+& \sigma^\prime((1-\eta)W_1 + b)[1-\sigma((1-\eta)W_1 + b)] \nonumber\\
&& [3\sigma((1-\eta)W_1 + b)-1]
\ee
\be
\frac{\partial^2 \mathcal{L}_{\rm AT}}{\partial b \partial W_1} &=& 
-\sigma^\prime(W_1 - b)[1-\sigma(W_1 - b)]\nonumber\\
&& [3\sigma(W_1 - b)-1] \nonumber\\
&+& \sigma^\prime(W_1 + b)[1-\sigma(W_1 + b)]\nonumber\\
&& [3\sigma(W_1 + b)-1] \nonumber\\
&+& \sigma^\prime((1-\eta)W_1 - b)[1-\sigma((1-\eta)W_1 - b)] \nonumber\\
&& [3\sigma((1-\eta)W_1 - b)-1](\eta-1) \nonumber\\
&+& \sigma^\prime((1-\eta)W_1 + b)[1-\sigma((1-\eta)W_1 + b)] \nonumber\\
&& [3\sigma((1-\eta)W_1 + b)-1](1-\eta)
\ee
\be
\frac{\partial^2 \mathcal{L}_{\rm AT}}{(\partial W_1)^2} &=& 
\sigma^\prime(W_1 - b)[1-\sigma(W_1 - b)]\nonumber\\
&& [3\sigma(W_1 - b)-1] \nonumber\\
&+& \sigma^\prime(W_1 + b)[1-\sigma(W_1 + b)]\nonumber\\
&& [3\sigma(W_1 + b)-1] \nonumber\\
&+& \sigma^\prime((1-\eta)W_1 - b)[1-\sigma((1-\eta)W_1 - b)] \nonumber\\
&& [3\sigma((1-\eta)W_1 - b)-1](1-\eta)^2 \nonumber\\
&+& \sigma^\prime((1-\eta)W_1 + b)[1-\sigma((1-\eta)W_1 + b)] \nonumber\\
&& [3\sigma((1-\eta)W_1 + b)-1](1-\eta)^2.
\ee

\begin{figure}[t]
  \centering
  \includegraphics[width=\linewidth]{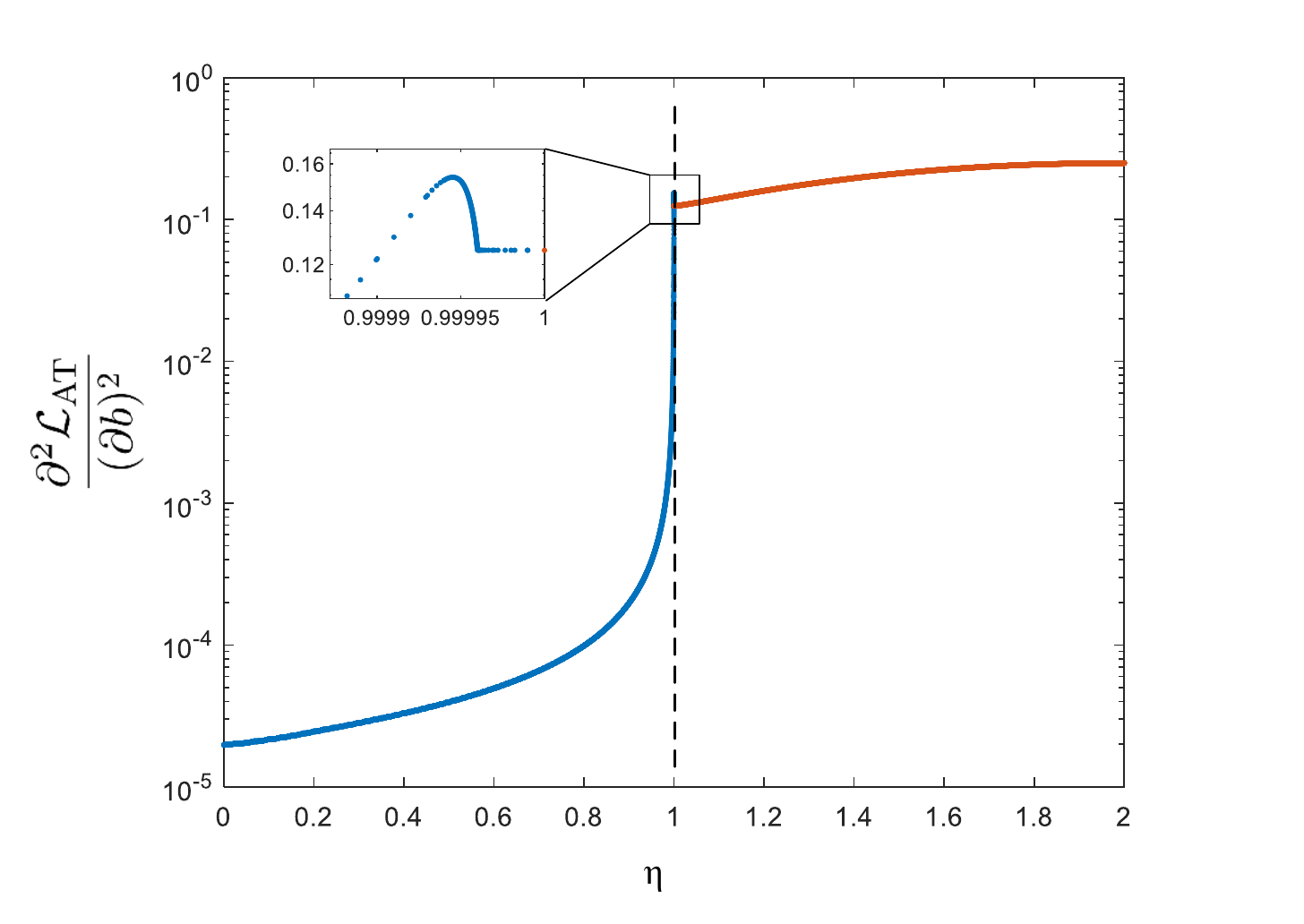} 
  \caption{The properties of the curve $\frac{\partial^2 \mathcal{L}_{\rm AT}}{(\partial b)^2}$ under best $W_1$, $b=0$ with the change of $\eta$. }
  \label{fig:l2at_all}
\end{figure}

The relationship between the second derivative can be summarized as follows. For all $\eta$, $\frac{\partial^2 \mathcal{L}_{\rm AT}}{(\partial b)^2} > \frac{\partial^2 \mathcal{L}_{\rm NT}}{(\partial b)^2} > 0$, and exist a maximum point around $\eta \leq 1$ as in Figure~\ref{fig:l2at_all}. For $\eta<1$, $b=0$ and $W_1>0$, it is similar to know that $\frac{\partial^2 \mathcal{L}_{\rm AT}}{(\partial W_1)^2} > \frac{\partial^2 \mathcal{L}_{\rm NT}}{(\partial W_1)^2} > 0$, with maximum point very close to $1$. But when $\eta \approx 1$, factor $1-\eta$ bring the second derivative $\frac{\partial^2 \mathcal{L}_{\rm AT}}{(\partial W_1)^2}$ diminish, roughly equal to that of NT.

\textbf{Thus, there exists a prescribed budget (lead the AEs tending towards the decision boundary) that accelerates the learning process (with larger second derivative of loss function on global minimum).}
\end{proof}

\subsection{Unrestricted Adversarial Training}

\begin{lemma}
Gradient based unconstrained AEs generating methods generate the AEs along $\hat{\mathbf{x}}_1$ axis. Written as,
\be
\hat{\mathbf{x}}^1_{\rm adv} &=& (1 - \eta_1) \hat{\mathbf{x}}^1 = (-1 + \eta_1, \hat{0})^\mathrm{T}, \\
\hat{\mathbf{x}}^2_{\rm adv} &=& (1 - \eta_2) \hat{\mathbf{x}}^2 = (1 - \eta_2, \hat{0})^\mathrm{T},
\ee
where $\eta_i \in [0,2], i=1,2,$ is the adversarial perturbation, and satisfy $\eta_1+\eta_2=2$, as AEs strictly on the current decision boundary.
\end{lemma}
\begin{proof}
For unrestricted AT, the optimizing objectives of AEs is to fool the network, with AEs beyond the current decision boundary. Because the current decision boundary usually do not equal to perfect decision boundary, the adversarial perturbation of $\hat{\mathbf{x}}^1$ and $\hat{\mathbf{x}}^2$ may not equal. Assume the AEs strictly on the current decision boundary, thus the adversarial perturbation satisfy $\eta_1+\eta_2=2$.
\end{proof}

\begin{prop}
For classifying two points $\mathbf{x}^1, \mathbf{x}^2$ with a $1$-layer perceptron model, 
unrestricted AT can not lead the model to the one with the best robustness.
\end{prop}
\begin{proof}

Similar as restricted AT, by minimize the loss function, we have,
\be
\frac{\partial \mathcal{L}}{\partial b} &=& 
[\sigma(W_1  (-1) + b) - 0]
\sigma^\prime(W_1  (-1) + b) \nonumber\\
&+& [\sigma(W_1  (-1+\eta_1) + b) - 0]
\sigma^\prime(W_1  (-1+\eta_1) + b) \nonumber\\
&+& [\sigma(W_1  (1) + b) - 1]
\sigma^\prime(W_1  (1) + b)  \nonumber\\
&+& [\sigma(W_1  (1-\eta_2) + b) - 1]
\sigma^\prime(W_1  (1-\eta_2) + b)  \nonumber\\
&+& \lambda{\rm sign}(b) = 0 \\
\frac{\partial \mathcal{L}}{\partial W_1} &=& 
[\sigma(W_1  (-1) + b) - 0]
-\sigma^\prime(W_1  (-1) + b) \nonumber\\
&+& [\sigma(W_1  (\eta_1-1) + b) - 0]  \nonumber\\
&& \sigma^\prime(W_1  (\eta_1-1) + b) (\eta_1-1) \nonumber\\
&+& [\sigma(W_1  (1) + b) - 1]
\sigma^\prime(W_1  (1) + b) (1)   \nonumber\\
&+& [\sigma(W_1  (1-\eta_2) + b) - 1]  \nonumber\\
&& \sigma^\prime(W_1  (1-\eta_2) + b) (1-\eta_2)   \nonumber\\
&+& \lambda{\rm sign}(W_1) = 0\\
\frac{\partial \mathcal{L}}{\partial W_i} &=& \lambda{\rm sign}(W_i), i\geq 2.
\ee
Bring in $\eta_1-1 = 1-\eta_2$,
\be
\frac{\partial \mathcal{L}}{\partial b} &=& 
[\sigma(b - W_1)]
\sigma^\prime(b - W_1) \nonumber\\
&+& [\sigma(W_1 + b) - 1]
\sigma^\prime(W_1 + b)  \nonumber\\
&+& [2\sigma(W_1  (1-\eta_2) + b) - 1]  \nonumber\\
&& \sigma^\prime(W_1  (1-\eta_2) + b)  \nonumber\\
&+& \lambda{\rm sign}(b) = 0\\
\frac{\partial \mathcal{L}}{\partial W_1} &=& 
[\sigma(b - W_1) - 0]
\sigma^\prime(b - W_1) (-1) \nonumber\\
&+& [\sigma(W_1  + b) - 1]
\sigma^\prime(W_1 + b)  \nonumber\\
&+& [2\sigma(W_1  (1-\eta_2) + b) - 1] \nonumber\\
&& \sigma^\prime(W_1  (1-\eta_2) + b) (1-\eta_2)   \nonumber\\
&+& \lambda{\rm sign}(W_1) = 0\\
\frac{\partial \mathcal{L}}{\partial W_i} &=& \lambda{\rm sign}(W_i), i\geq 2.
\ee
The second derivative of $\frac{\partial^2 \mathcal{L}_{\rm AT}}{(\partial b)^2}$ of unrestricted AT as follows,
\be
\frac{\partial^2 \mathcal{L}_{\rm AT}}{(\partial b)^2} &=& 
\sigma^\prime(W_1 - b)[1-\sigma(W_1 - b)] \nonumber\\
&& [3\sigma(W_1 - b)-1] \nonumber\\
&+& \sigma^\prime(W_1 + b)[1-\sigma(W_1 + b)] \nonumber\\
&& [3\sigma(W_1 + b)-1] \nonumber\\
&+& \sigma^\prime((1-\eta_1)W_1 - b)[1-\sigma((1-\eta_1)W_1 - b)] \nonumber\\
&& [3\sigma((1-\eta_1)W_1 - b)-1] \nonumber\\
&+& \sigma^\prime((1-\eta_2)W_1 + b)[1-\sigma((1-\eta_2)W_1 + b)] \nonumber\\
&& [3\sigma((1-\eta_2)W_1 + b)-1]. \label{eq:uat:3}
\ee

If $\eta_1-1 = 1-\eta_2 = -\frac{b}{W_1}$, on the current decision boundary, the ``$1-\eta_2$''(AEs) term come to zero, then
\be
\frac{\partial \mathcal{L}}{\partial b} &=& 
[\sigma(b - W_1)]
\sigma^\prime(b - W_1) \nonumber\\
&+& [\sigma(W_1 + b) - 1]
\sigma^\prime(W_1 + b)  \nonumber\\
&+& \lambda{\rm sign}(b) = 0 \label{eq:uat:1}\\
\frac{\partial \mathcal{L}}{\partial W_1} &=& 
[\sigma(b - W_1) - 0]
\sigma^\prime(b - W_1) (-1) \nonumber\\
&+& [\sigma(W_1  + b) - 1]
\sigma^\prime(W_1 + b)  \nonumber\\
&+& \lambda{\rm sign}(W_1) = 0 \label{eq:uat:2}\\
\frac{\partial \mathcal{L}}{\partial W_i} &=& \lambda{\rm sign}(W_i), i\geq 2.
\ee
Although the loss function has same minimize as NT, but the ``$1-\eta_2$''(AEs) term in the gradient equal to $0$, the gradient is almost half of which of NT. \textbf{Thus, unrestricted AT can not ensure better result than NT.}

\begin{figure}[t]
  \centering
  \includegraphics[width=\linewidth]{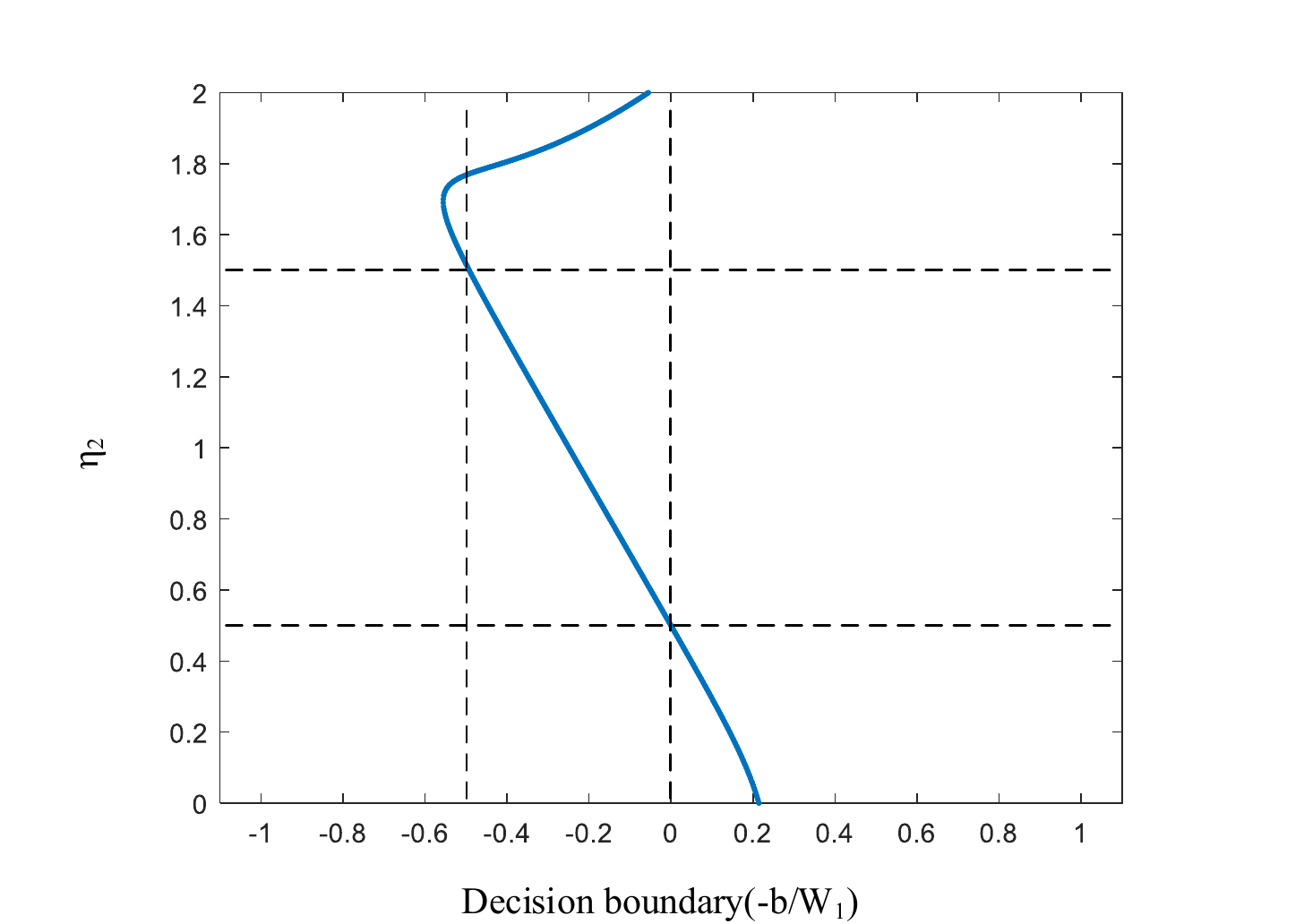}
  \includegraphics[width=\linewidth]{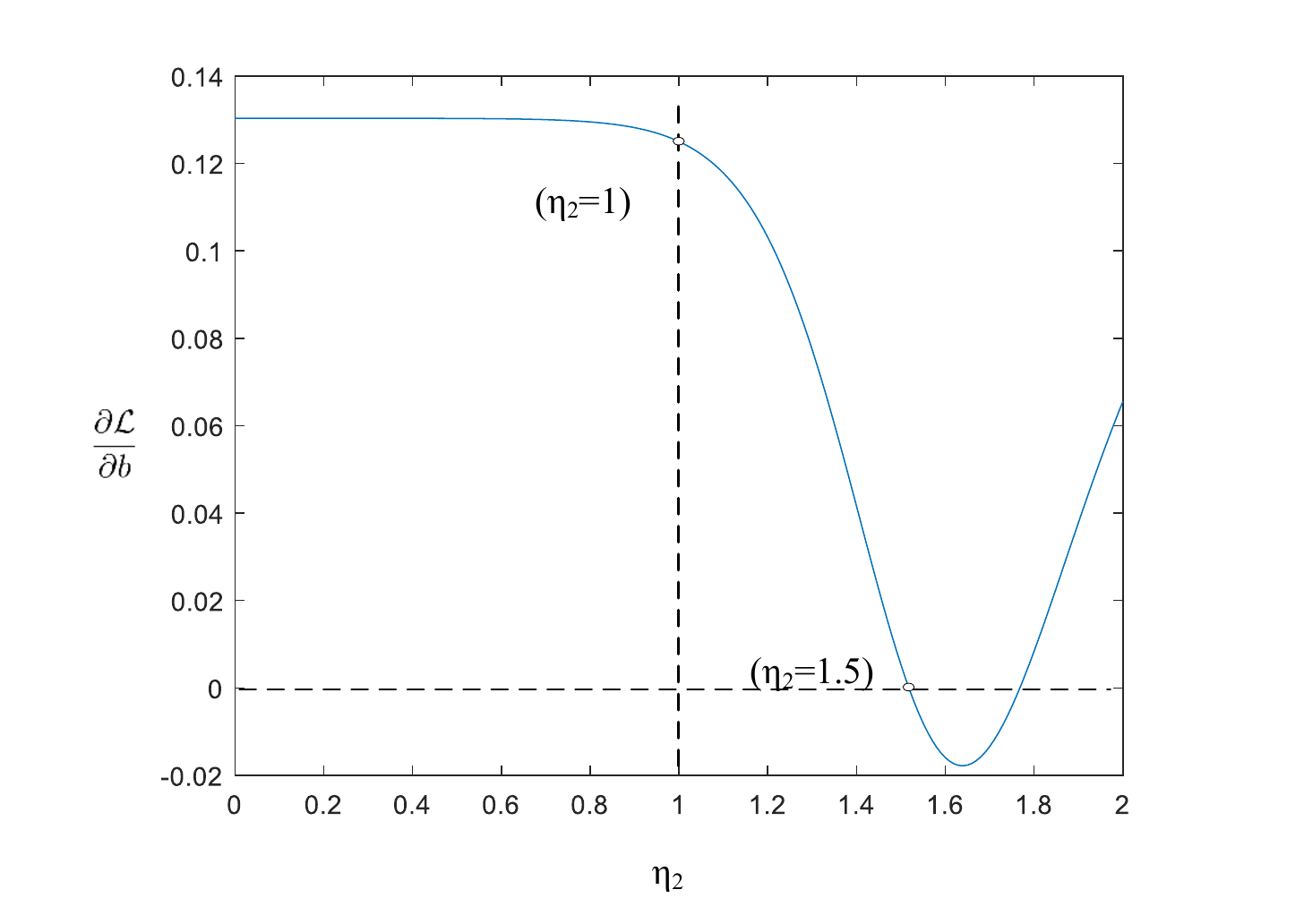}
  \caption{(Upper) The decision boundary varying with $\eta_2$ with $W_1 = 5$ and $\eta_1 = 0.5$.
  (Lower) $\frac{\partial \mathcal{L}}{\partial b}$ with varying $\eta_2$ with $W_1 = 5$, $b=2.5$ and $\eta_1 = 0.5$.
  }
  \label{fig:uat_up}
\end{figure}

Otherwise during unrestricted AT, there must exist a state that $b \neq 0$ (not convergence to its best). 
In order to analyze the effects of AEs crossing the decision boundary, we test different $\eta_2$ as in Figure~\ref{fig:uat_up}(Upper)~\&~(Lower), different $\eta_2$ may lead different converge decision boundary with different $\frac{\partial \mathcal{L}}{\partial b}$. By fixing $\eta_1$, the convergence limit of $b$ changes with parameters $\eta_2$. Due to assume the current decision boundary located at $-0.5$, AEs in unrestricted AT with $\eta_2 = 1.5$ lead $\frac{\partial \mathcal{L}}{\partial b}$ very close to $0$, almost impossible to improve the decision boundary effectively. Smaller $\eta_2$ can produce more improvement until $\eta_2=0.5$, achieve restricted AT. $\frac{\partial \mathcal{L}}{\partial b}$ reach its maximum when $\eta_2 = 0$. If $\eta_2 \leq 1$, $\frac{\partial \mathcal{L}}{\partial b}$ almost reach its maximum. This means unrestricted AT still has much room for improvement.
\textbf{Thus, the ``$1-\eta_2$'' term will influence the results and deviating the model to the one with the best robustness.}

\end{proof}

\subsection{Blind Adversarial Training}

\begin{prop}
For classifying two points $\mathbf{x}^1, \mathbf{x}^2$ with a $1$-layer perceptron model, 
BAT also can lead the model to the one with the best robustness, and has the same convergence property as restricted AT with best budget.
\end{prop} 
\begin{proof}
The proposed BAT can achieve such a good results, due to the cutoff-scale strategy therein.

\textbf{Cutoff: }
Cutoff using average norm of AEs as the budget, to prevent AEs beyond the perfect decision boundary. In numerical calculation, it is impossible to exact know the perfect decision boundary, but the average of AEs, to some extent, reflects the nature and give a path to potential guess of perfect decision boundary. 

As in this perceptron model in Figure~\ref{fig:uat_up}~(Lower), AEs be on the current decision boundary $\eta_2=1.5$, lead $\frac{\partial \mathcal{L}}{\partial b} \sim 0$, make the unrestricted AT worse. But when $\eta_2 <1$, $\frac{\partial \mathcal{L}}{\partial b}$ reach maximum, the model updated quickly. Cutoff used the average of AEs satisfy $\varepsilon = \frac{\eta_1 + \eta_2}{2} = 1$, equivalent to reset the budget to $\eta_1 = 0.5, \eta_2 = 1$, can enlarge the $\frac{\partial \mathcal{L}}{\partial b}$ and lead the model converge faster as in Figure~\ref{fig:uat_up}. \textbf{Thus, cutoff can make BAT algorithm achieving the similar effect as restricted AT without using prescribed budget.}

\textbf{Scale: }
As in Figure~\ref{fig:l2at_all}, the maximum second derivative of loss function appears at a little distance from the decision boundary, scale can lead the AEs from the decision boundary to such position. Meanwhile, most unconstrained AEs may produce AEs beyond the current decision boundary, scale can pull these AEs back into the decision boundary, make the algorithm convergence faster. \textbf{Thus, scale can make BAT algorithm achieving the similar effect as restricted AT with best budget.} 

\textbf{Convergence property: }
Once we obtain the ideal model (the decision boundary is actually the centerline of the data with different labels) and optimal DeepFool AEs (just slightly beyond the decision boundary), $\mathbb{E}||\delta(\mathbf{x})||$ can represent the mean distance between the data and the decision boundary.
Thus, the cutoff will directly cut off the AEs with relatively large strength, while the scale will ensure that the AEs will not extend beyond the decision boundary, especially for the AEs with small strength. They both have little influence on the ideal model. For example, as shown in Figure \ref{fig:tcc_app} (d), the cutoff and scale will maintain the AEs not to affect the model's updating.
\textbf{Thus, it is available to know the result of final convergence is the result of restricted AT with best budget, has the same convergence property.}

\end{proof}

So in conclusion, the proposed BAT can perform the best facing two points $\mathbf{x}^1, \mathbf{x}^2$ classification problem with a $1$-layer perceptron model. We can extend this conclusion to more general problems, and conjecture that for general classification neural network, BAT can dynamically adjust a nonuniform budget, seek to provide a path to potentially guess the perfect decision boundary, and finally reach the model with best robustness.

\end{document}